\documentclass[11pt]{article}

\usepackage{times, fullpage}
\usepackage{graphicx} % more modern
\usepackage{subfigure} 

\usepackage{amsthm,amsfonts,amsmath,amssymb,color,float,graphicx,verbatim}
\usepackage{natbib}
\usepackage{algorithm,algorithmic}

\usepackage{hyperref}
\hypersetup{
	hidelinks = true
}

\usepackage{mathtools}
\usepackage{bbm}
\usepackage{csquotes}

\newtheorem{theorem}{Theorem}

\newtheorem{lemma}{Lemma}
\newtheorem{corollary}{Corollary}
\newtheorem{definition}{Definition}
\newtheorem{assumption}{Assumption}

\newcommand{\ceil}[1]{\left\lceil #1 \right\rceil}
\newcommand{\abs}[1]{\left| #1 \right|}

\newcommand{\reals}{\mathbb{R}}
\newcommand{\E}{\mathbb{E}}

\newcommand{\relu}[1]{\left[ #1 \right]_+}

\newcommand{\set}[1]{\left\lbrace#1\right\rbrace}

\newcommand{\p}[1]{\left( #1 \right)}
\newcommand{\pcc}[1]{\left[ #1 \right]}

\newcommand{\bbr}{\mathbb{R}}

\newcommand{\bx}{\mathbf{x}}
\newcommand{\bw}{\mathbf{w}}

\newcommand{\bb}{\mathbf{b}}
\newcommand{\bu}{\mathbf{u}}
\newcommand{\bv}{\mathbf{v}}
\newcommand{\bz}{\mathbf{z}}

\newcommand{\by}{\mathbf{y}}
\newcommand{\bn}{\mathbf{n}}

\newcommand{\bsigma}{\boldsymbol{\sigma}}

\newcommand{\Ocal}{\mathcal{O}}

\newcommand{\Xcal}{\mathcal{X}}

\newcommand{\norm}[1]{\left\|#1\right\|}
\newcommand{\inner}[1]{\left\langle#1\right\rangle}

\newcommand{\ind}[1]{\mathbbm{1}\left\lbrace#1\right\rbrace}

\newcommand{\secref}[1]{Sec.~\ref{#1}}
\newcommand{\subsecref}[1]{Subsection~\ref{#1}}
\newcommand{\figref}[1]{Fig.~\ref{#1}}
\renewcommand{\eqref}[1]{Eq.~(\ref{#1})}
\newcommand{\lemref}[1]{Lemma~\ref{#1}}
\newcommand{\corollaryref}[1]{Corollary~\ref{#1}}
\newcommand{\thmref}[1]{Thm.~\ref{#1}}

\begin{document}
	
\title{Depth-Width Tradeoffs in Approximating Natural Functions 
	with Neural Networks}

\author{Itay Safran\\Weizmann Institute of Science\\\texttt{itay.safran@weizmann.ac.il}\and
	Ohad Shamir
	\\Weizmann Institute of Science\\{\texttt{ohad.shamir@weizmann.ac.il}}
}
\date{}
\maketitle

	\begin{abstract}
		We provide several new depth-based separation results for feed-forward 
		neural networks, proving that various types of simple and natural 
		functions can be better approximated using deeper networks than 
		shallower ones, even if the shallower networks are much larger. 
		This includes indicators of balls and ellipses; non-linear functions 
		which are 
		radial with respect to the $L_1$ norm; and smooth non-linear functions. 
		We also show that these gaps can be observed experimentally: 
		Increasing the depth indeed allows better learning than increasing 
		width, when training neural networks to learn an indicator of a unit 
		ball.
	\end{abstract}
	
	\section{Introduction}
	Deep learning, in the form of artificial neural networks, has seen a 
	dramatic resurgence in the past recent years, achieving great performance 
	improvements in various fields of artificial intelligence such as computer 
	vision and speech recognition. While empirically successful, our 
	theoretical	understanding of deep learning is still limited at best.
	
	An emerging line of recent works has studied the \emph{expressive power} of 
	neural networks: What functions can and cannot be represented by networks 
	of a given architecture (see related work section below). A particular 
	focus has been the trade-off between the network's width and depth: On the 
	one hand, it is well-known that large enough networks of depth $ 2 $ can 
	already approximate any continuous target function on  
	$ \pcc{0,1}^d $ to arbitrary accuracy
	\citep{cybenko1989approximation,hornik1991approximation}. On the other 
	hand, it has long been evident that deeper networks tend to perform 
	better than shallow ones, a phenomenon supported by the intuition that 
	depth, providing compositional expressibility, is necessary for efficiently 
	representing some functions. Indeed, recent empirical evidence suggests 
	that even at large depths, deeper networks can offer benefits over 
	shallower networks \citep{he2015deep}.
	
	To demonstrate the power of depth in neural networks, a clean and precise 
	approach is to prove the existence of functions which can be expressed (or 
	well-approximated) by moderately-sized networks of a given depth, yet 
	cannot be approximated well by shallower networks, even if their size is 
	much larger. However, the mere existence of such functions is not enough: 
	Ideally, we would like to show such depth separation results using 
	\emph{natural}, interpretable functions, of the type we may expect neural 
	networks to successfully train on. Proving that 
	depth is necessary for such functions can give us a clearer and more useful 
	insight into what various neural network architectures can and cannot 
	express in practice. 
	
	In this paper, we provide several contributions to this emerging line of 
	work. We focus on standard, vanilla feedforward networks (using some fixed 
	activation function, such as the popular ReLU), and measure expressiveness 
	directly in terms of approximation error, defined as the expected squared 
	loss with respect to some distribution over the input domain. In this 
	setting, we 
	show the following: 
	\begin{itemize}
	\item We prove that the indicator of the Euclidean unit ball, 
	$\bx\mapsto \ind{\norm{\bx}\leq 1}$ in $\reals^d$, which can be easily 
	approximated to accuracy $\epsilon$ using a 3-layer network with 
	$\Ocal(d/\epsilon)$ neurons, cannot be approximated to an accuracy higher 
	than $\Ocal(1/d^4)$  using a 2-layer network, unless its width is 
	exponential in $d$. In fact, we show the same result more generally, for 
	any	indicator of an ellipsoid $\bx\mapsto \ind{\norm{A\bx+\mathbf{b}}\leq 
	r}$ (where $A$ is a non-singular matrix and $\bb$ is a vector). The proof 
	is based on a reduction from the main result of \cite{eldan2016power}, 
	which shows a separation between 2-layer and 3-layer networks using a  
	more complicated and less natural radial function.
	\item We show that this depth/width trade-off can also be observed 
	experimentally: 
	Specifically, that
	the indicator of a unit ball can be learned quite well using a 3-layer 
	network, using standard backpropagation,
	but learning the same function with a 2-layer network (even if much 
	larger) is significantly more difficult. 
	Our theoretical result indicates that this gap in performance is due to approximation error issues.
	This experiment also highlights the fact that our separation result is for 
	a natural function that is not just well-approximated by some 3-layer 
	network, but can also be learned well from data using standard methods. 
	\item We prove that \emph{any} $L_1$ radial function $\bx\mapsto 
	f(\norm{\bx}_1)$, where $\bx\in\reals^d$ and $f:\reals\rightarrow\reals$ is 
	piecewise-linear, cannot be approximated to accuracy $\epsilon$ by a depth 
	2 ReLU network of width less than 
	$\tilde{\Omega}(\min\{1/\epsilon,\exp(\Omega(d))\})$. In 
	contrast, such functions can be represented \emph{exactly} by 3-layer ReLU 
	networks. 
	\item Finally, we prove that any member of a wide family of non-linear and 
	twice-differentiable functions (including for instance $x\mapsto x^2$ in 
	$[0,1]$), which can be approximated to accuracy $\epsilon$ using ReLU 
	networks of depth and width $\Ocal(\text{poly}(\log(1/\epsilon)))$, cannot 
	be approximated to similar accuracy by constant-depth ReLU networks, unless 
	their width is at least $\Omega(\text{poly}(1/\epsilon))$. We note that a similar 
	result appeared online concurrently and independently of ours in  
	\cite{yarotsky2016error,liang2016deep}, but the setting is a bit different 
	(see related work below for more details). 
	\end{itemize}
	
	\subsubsection*{Related Work}
	
	The question of studying the effect of depth in neural network has received 
	considerable attention recently, and studied under various settings. Many 
	of these works consider a somewhat different setting than ours, and hence 
	are not directly comparable. These include networks which are not 
	plain-vanilla ones (e.g. 
	\cite{cohen2016expressive,delalleau2011shallow,martens2014expressive}), 
	measuring quantities other than 
	approximation error (e.g. \cite{bianchinicomplexity,poole2016exponential}), 
	focusing only on approximation upper bounds 
	(e.g. \cite{shaham2016provable}), or measuring approximation 
	error in terms of $L_{\infty}$-type bounds, i.e. 	
	$\sup_{\bx}|f(\bx)-\tilde{f}(\bx))|$ rather than $L_2$-type bounds		
	$\E_{\bx}(f(\bx)-\tilde{f}(\bx))^2$ (e.g. 		
	\cite{yarotsky2016error,liang2016deep}). We note that the latter 
	distinction is important: Although $L_{\infty}$ bounds are more common in 
	the 
	approximation theory literature, $L_2$ bounds are more natural in the 
	context of statistical machine learning problems (where we care about the 
	expected loss over some distribution). Moreover, $L_2$ approximation lower 
	bounds are stronger, in 
	the sense that an $L_2$ lower bound easily 
	translates to a lower bound on $L_{\infty}$ lower bound, but not vice 
	versa\footnote{To give a 
	trivial example, ReLU networks always express continuous functions, and 
	therefore can never approximate a discontinuous 
	function such as $x\mapsto\ind{x\geq 0}$ in an $L_{\infty}$ sense, yet can 
	easily approximate it in an $L_{2}$ sense given any continuous 
	distribution.}.
	
	A noteworthy paper in the same setting as ours is 
	\cite{telgarsky2016benefits}, which proves a separation 
	result between the expressivity of ReLU networks of depth $ k $ and depth $ 
	o\p{k/\log\p{k}} $ (for any $k$). This holds even for 
	one-dimensional functions, where a depth $k$ network is shown to realize a 
	saw-tooth function with $\exp(\Ocal(k))$ oscillations, whereas any network 
	of depth $o\p{k/\log\p{k}} $ would require a width super-polynomial 
	in $k$ to approximate it by more than a constant. In fact, we ourselves 
	rely on this construction in the proofs of our results in 
	\secref{sec:C2}. On the flip side, in our paper 
	we focus on separation in terms of the accuracy or dimension, rather than a 
	parameter $k$. Moreover, the construction there relies on a highly  
	oscillatory function, with Lipschitz constant exponential in $ k $ almost 
	everywhere. In 
	contrast, in our paper we focus on simpler functions, of the type that 
	are likely to be learnable from data 
	using standard methods. 
	
	Our separation results in \secref{sec:C2} (for smooth non-linear functions)
	are closely related 
	to those of \cite{yarotsky2016error,liang2016deep}, which appeared 
	online concurrently and independently of our work, and the proof ideas are 
	quite similar. However, these papers focused on $L_{\infty}$ 
	bounds rather than $L_{2}$ bounds. Moreover, \cite{yarotsky2016error} 
	considers a class of functions different than ours in their 
	positive results, and \cite{liang2016deep} consider networks employing a 
	mix of ReLU and threshold activations, whereas we consider a purely ReLU 
	network. 
	
	Another relevant and insightful work is \cite{poggio2016and}, which 
	considers width vs. depth and provide general results on expressibility of 
	functions with a compositional nature. However, the focus there is on 
	worse-case approximation over general classes of functions, rather than 
	separation results in terms of	specific functions as we do here, and the 
	details and setting is somewhat	orthogonal to ours.

	\section{Preliminaries}\label{sec:prelim}
	
	In general, we let bold-faced letters such as $\bx=(x_1,\ldots,x_d)$ denote 
	vectors, and capital letters denote matrices or probabilistic events. 
	$\norm{\cdot}$ denotes the Euclidean norm, and $\norm{\cdot}_1$ the 
	$1$-norm. $\ind{\cdot}$ denotes the indicator function. We use the standard 
	asymptotic notation $\Ocal(\cdot)$ and $\Omega(\cdot)$ to hide constants, 
	and $\tilde{\Ocal}(\cdot)$ and $\tilde{\Omega}(\cdot)$ to hide constants 
	and factors logarithmic in the problem parameters.
	
	\textbf{Neural Networks.} We consider feed-forward neural 
	networks, computing functions from 
	$\reals^d$ to $\reals$. The network is composed of layers of neurons, where 
	each neuron computes a function of the form 
	$\bx\mapsto\sigma(\bw^{\top}\bx+b)$, where $\bw$ is a weight vector, $b$ is 
	a bias term and $\sigma:\reals\mapsto\reals$ is a non-linear activation 
	function, such as the ReLU function $\sigma(z) = [z]_+ = \max\{0,z\}$. 
	Letting $\sigma(W\bx+\bb)$ be a shorthand for 
	$\p{\sigma(\bw_1^{\top}\bx+b_1),\ldots,\sigma(\bw_n^{\top}\bx+b_n)}$, we 
	define a layer of $n$ neurons as $\bx\mapsto\sigma(W\bx+\bb)$. By denoting 
	the output of the $i^{\text{th}}$ layer as $O_i$, 
	we can define a network of arbitrary depth recursively by 
	$O_{i+1}=\sigma(W_{i+1}O_i+\bb_{i+1})$, where $W_i,\bb_i$ 
	represent the matrix of weights and bias of the $i^{\text{th}}$ layer, 
	respectively. Following a standard convention for multi-layer networks, the 
	final layer $h$ is a purely linear function with no bias, i.e. 
	$O_h=W_h\cdot O_{h-1}$. We define the \emph{depth} of the network as the 
	number of layers $l$, and denote the number of neurons $n_i$ in the 
	$i^{\text{th}}$ layer as the \emph{size} of the layer. We define the 
	\emph{width} of a network as $\max_{i\in\set{1,\dots,l}}n_i$. Finally, a 
	\emph{ReLU network} is a neural network where all the non-linear 
	activations are the ReLU function. We use ``2-layer'' and ``3-layer'' to 
	denote networks of depth 2 and 3. In particular, in our notation a 2-layer 
	ReLU network has the form
	\[
	\bx \mapsto \sum_{i=1}^{n_1}v_i\cdot [\bw_i^\top \bx+b_i]_+
	\]
	for some parameters $v_1,b_1,\ldots,v_{n_1},b_{n_1}$ and $d$-dimensional 
	vectors 
	$\bw_1,\ldots,\bw_{n_1}$. Similarly, a 3-layer ReLU network has the form
	\[
	\sum_{i=1}^{n_2}u_i\left[\sum_{j=1}^{n_1}v_{i,j}\left[\bw_{i,j}^\top\bx
	+b_{i,j}\right]_++c_{i}\right]_+
	\]
	for some parameters $\{u_i,v_{i,j},b_{i,j},c_i,\bw_{i,j}\}$.

	\textbf{Approximation error.} Given some function $f$ on a domain $\Xcal$ 
	endowed with some probability distribution (with density function $\mu$), 
	we define the 
	quality of its approximation by some other function $\tilde{f}$ as 
	$\int_{\Xcal} (f(\bx)-\tilde{f}(\bx))^2 \mu(\bx)d\bx = 
	\E_{\bx\sim\mu}[(f(\bx)-\tilde{f}(\bx))^2]$. We refer to this as 
	approximation in the $L_2$-norm sense. In one of our results 
	(\thmref{thm:main_ubound}), we also 
	consider approximation in the $L_{\infty}$-norm sense, defined as
	$\sup_{\bx\in\Xcal}|f(\bx)-\tilde{f}(\bx)|$. Clearly, this 
	upper-bounds the (square root of the) $L_2$ approximation error defined 
	above, so as discussed in the introduction, lower bounds on the $L_2$ 
	approximation error (w.r.t. any distribution) are stronger than lower 
	bounds on the $L_{\infty}$ approximation error. 
	
	\section{Indicators of $L_2$ Balls and Ellipsoids}\label{sec:indicators}
	
	We begin by considering one of the simplest possible function classes on 
	$\reals^d$, namely indicators of $L_2$ balls (and more generally, 
	ellipsoids). The ability to compute such functions is necessary for many 
	useful primitives, for example determining if the distance between two 
	points in Euclidean space is below or above some threshold (either with 
	respect to the Euclidean distance, or a more general Mahalanobis distance). 
	In this section, we show a depth separation result for such functions:	
	Although they can be easily approximated with 3-layer networks, no 2-layer 
	network can approximate it to high accuracy w.r.t. any distribution, unless 
	its width is 
	exponential in the dimension. This is formally stated in the following 
	theorem: 	
	\begin{theorem}[Inapproximability with 2-layer 
	networks]\label{thm:ellipsoidlowbound}
		The following holds for some positive universal constants 
		$c_1,c_2,c_3,c_4$, and any network employing an activation function 
		satisfying Assumptions 1 and 2 in \citet{eldan2016power}: For any 
		$d>c_1$, and any non-singular matrix $ 
		A\in\bbr^{d\times d} $, $ 
		\bb\in\bbr^d 
		$ and $r\in (0,\infty)$, there exists a continuous probability distribution 
		$\gamma$ on $\bbr^d$, such that for any function $ g $ computed by a 
		2-layer network of width at most $c_3 \exp(c_4 d)$, and for the 
		function $ f(\bx)=\ind{\norm{A\bx+\bb}\le r} $, we have 
		\begin{equation*}
			\int_{\bbr^d}\p{f(\bx)-g(\bx)}^2\cdot\gamma(\bx)d\bx~\ge~ 
			\frac{c_2}{d^4}.
		\end{equation*} 
	\end{theorem}
	We note that the assumptions from \cite{eldan2016power} are very mild, 
	and apply to all 
	standard activation functions, including ReLU, sigmoid and threshold. 
	
	The formal proof of \thmref{thm:ellipsoidlowbound} (provided below) 
	is based on a reduction from 
	the main result of \cite{eldan2016power}, which shows the existence of a 
	certain radial function (depending on the input $\bx$ only through its 
	norm) and a probability distribution which cannot be expressed by a 2-layer 
	network, whose width is less than exponential in the dimension $d$ to more 
	than constant accuracy. A closer look at the proof reveals that this 
	function (denoted as $\tilde{g}$) can be expressed as a sum of 
	$\Theta(d^2)$ 
	indicators of $L_2$ balls of various radii. We argue that if we could have 
	accurately approximated a given $L_2$ ball indicator with respect to all 
	distributions, then we could have approximated all the indicators whose sum 
	add up to $\tilde{g}$, and hence reach a contradiction. By a linear 
	transformation argument, we show the same contradiction would have occured 
	if we could have approximated the indicators of an non-degenerate ellipse 
	with respect to any distribution. The formal proof is provided below:
	
	\begin{proof}[Proof of \thmref{thm:ellipsoidlowbound}]
		Assume by contradiction that for $f$ as described in the theorem, and for 
	any distribution $\gamma$, there exists a two-layer network 
	$\tilde{f}_{\gamma}$ of width at most $c_3\exp(c_4d)$, such that
	\[
	\int_{\bx\in\bbr^d} \p{f(\bx)-\tilde{f}_{\gamma}(\bx)}^2\gamma(\bx)d\bx
	~\le~
	\epsilon~\le \frac{c_2}{d^4}.
	\]
	
	Let $\hat{A}$ and $\hat{\bb}$ be a $d\times d$ non-singular matrix and 
	vector respectively, to be determined later. We 
	begin by performing a change of variables, $\by= \hat{A}\bx+\hat{\bb}\iff 
	\bx=\hat{A}^{-1}(\by-\hat{\bb})$, $ d\bx=\abs{\det\p{\hat{A}^{-1}}}\cdot 
	d\by $, 
	which yields
	\begin{align}
		\int_{\by\in\reals^d} &
		\p{f\p{\hat{A}^{-1}\p{\by-\hat{\bb}}}-\tilde{f}_{\gamma}\p{\hat{A}^{-1}\p{\by-\hat{\bb}}}}^2\notag\\
		&\cdot\gamma\p{\hat{A}^{-1}\p{\by-\hat{\bb}}}\cdot
		\abs{\det\p{\hat{A}^{-1}}}\cdot d\by~\le~ \epsilon.\label{eq:gamma_change_of_var}
	\end{align}
	
	In particular, let us choose the distribution $\gamma$ defined as $ 
	\gamma(\bz) = |\det(\hat{A})|\cdot 
	\mu(\hat{A}\bz+\hat{\bb})$, where $\mu$ is the (continuous) distribution 
	used in the 
	main 
	result of \cite{eldan2016power}  (note that $ \gamma $ is indeed a 
	distribution, since 
	$\int_{\bz}\gamma\p{\bz}=(\det(\hat{A}))\int_{\bz} 
	\mu(\hat{A}\bz+\hat{\bb})d\bz$, which by the change of variables $ 
	\bx=\hat{A}\bz+\hat{\bb} 
	$, $d\bx=|\det(\hat{A})|d\bz $ equals 
	$\int_{\bx}\mu(\bx)d\bx=1$). 
	Plugging the definition of $ \gamma $ in \eqref{eq:gamma_change_of_var}, 
	and using the fact that 
	$|\det(\hat{A}^{-1})|\cdot|\det(\hat{A})| = 
	1$, we 
	get
	\begin{align}
		\int_{\by\in\reals^d} 
		&\p{f\p{\hat{A}^{-1}\p{\by-\hat{\bb}}}-\tilde{f}_{\gamma}\p{\hat{A}^{-1}\p{\by-\hat{\bb}}}}^2\notag\\
		&~~~~~~~~~~~~~~~~~~~~~~~~~~~~~~~~~~~~~~\cdot\mu\p{\by}d\by~\le~ \epsilon.
	\end{align}
	Letting $z>0$ be an arbitrary parameter, we now pick
	$\hat{A}=\frac{z}{r}A$ and $\hat{\bb} = 
	\frac{z}{r}\bb$. Recalling the definition of $f$ as 
	$\bx\mapsto\ind{\norm{A\bx+\bb}\leq r}$, we get that
	\begin{align}
	\int_{\by\in\bbr^d} 
	&\p{\ind{\norm{\by}\leq 
			z}-\tilde{f}_{\gamma}\p{\frac{r}{z}A^{-1}
			\p{\by-\frac{z}{r}\bb}}}^2\notag\\
	&~~~~~~~~~~~~~~~~~~~~~~~~~~~~~~~~~~~~~~\cdot\mu\p{\by}d\by~\le~ \epsilon.	
	\end{align}
	Note that 
	$\tilde{f}_{\gamma}\p{\frac{r}{z}A^{-1}\p{y-\frac{z}{r}\bb}}$ 
	expresses a 
	2-layer network composed with a linear transformation of the input, and 
	hence 
	can be expressed in turn by a 2-layer network (as we can absorb the linear 
	transformation into the parameters of each neuron in the first layer). 
	Therefore, letting 
	$\norm{f}_{L_2(\mu)}=\sqrt{\int_{\by}f^2(\by)d\by}$ denote the norm in 
	$L_2(\mu)$ function space, we showed the following: For 
	any $z>0$, there exists a 2-layer network $\tilde{f}_{z}$ such 
	that 
	\begin{equation}\label{eq:keyapprx}
		\norm{
			\p{\ind{\norm{\cdot}\leq 
					z}-\tilde{f}_{z}\p{\cdot}}}_{L_2(\mu)}~\leq~\sqrt{\epsilon}. 
	\end{equation}
	With this key result in hand, we now turn to complete the proof. We 
	consider the function $\tilde{g}$ from \cite{eldan2016power}, for which it 
	was proven that no 2-layer network can approximate it w.r.t. $\mu$ to 
	better than constant accuracy, unless its width is exponential in the 
	dimension $d$. In particular $\tilde{g}$ can be written as
	\[
	\tilde{g}(\bx) = \sum_{i=1}^{n}\epsilon_i \cdot\ind{\norm{\bx}\in 
		[a_i,b_i]},
	\]
	where $[a_i,b_i]$ are disjoint intervals, $\epsilon_i\in \{-1,+1\}$, and 
	$n=\Theta(d^2)$ where $d$ is the dimension. Since $\tilde{g}$ can also be 
	written as
	\[
	\sum_{i=1}^{n}\epsilon_i 
	\p{\ind{\norm{\bx}\leq b_i}-\ind{\norm{\bx}\leq 
			a_i}},
	\]
	we get by \eqref{eq:keyapprx} and the triangle inequality that
	\begin{align*}
		&\norm{\tilde{g}(\cdot)-\sum_{i=1}^{n}\epsilon_i\cdot 
			(\tilde{f}_{b_i}(\cdot)-\tilde{f}_{a_i}(\cdot)}_{L_2(\mu)}\\
		&~\leq~ \sum_{i=1}^{n}|\epsilon_i|\left(\norm{\p{\ind{\norm{\cdot}\leq 
						b_i}-\tilde{f}_{b_i}}}_{L_2(\mu)}\right.\\
		&~~~~~~~~~~~~~~~~~~~~~~~~+\left.\norm{\ind{\norm{\cdot}\leq 
					a_i}-\tilde{f}_{a_i}(\cdot)}_{L_2(\mu)}\right)\\
		&~\leq~
		2n\sqrt{\epsilon}.
	\end{align*}
	However, since a linear combination of $2n$ 2-layer neural networks of 
	width at most $w$ is still a 2-layer network, of width at most $2nw$, we 
	get that
	$\sum_{i=1}^{n}\epsilon_i\cdot 
	(\tilde{f}_{b_i}(\cdot)-\tilde{f}_{a_i}(\cdot))$ is a 2-layer network, of 
	width at most $\Theta(d^2)\cdot c_3\exp(c_4 d)$, which approximates 
	$\tilde{g}$ to an accuracy of less than 
	$2n\sqrt{\epsilon}=\Theta(d^2)\cdot \sqrt{c_2/d^4} = \Theta(1)\cdot 
	\sqrt{c_2}$. Hence, by picking $c_2,c_3,c_4$ sufficiently small, we get a 
	contradiction to the result of \cite{eldan2016power}, that no 2-layer 
	network of width smaller than $c\exp(cd)$ (for some constant $c$) can 
	approximate $\tilde{g}$ to more than constant accuracy, for a sufficiently 
	large dimension $d$. 
	\end{proof}

	To complement \thmref{thm:ellipsoidlowbound}, we also show that such 
	indicator functions can be easily 
	approximated with 3-layer networks. The argument is quite simple: Using an 
	activation such as ReLU or Sigmoid, we can use one layer to approximate any 
	Lipschitz continuous function on any bounded interval, and in particular 
	$x\mapsto x^2$. Given a vector $\bx\in \reals^d$, we can apply this 
	construction on each coordinate $x_i$ seperately, hence approximating 
	$\bx\mapsto \norm{\bx}^2=\sum_{i=1}^{d}x_i^2$. Similarly, we can 
	approximate $\bx\mapsto \norm{A\bx+\bb}$ for arbitrary fixed matrices $A$ and 
	vectors $\bb$. Finally, with a 3-layer network, we can use the second layer 
	to compute a continuous approximation to the threshold function $z\mapsto 
	\ind{z\leq r}$. Composing these two layers, we get an arbitrarily good 
	approximation to 
	the function $\bx\mapsto \ind{\norm{A\bx+\bb}\le r}$ w.r.t. any continuous 
	distribution, with the network size scaling polynomially with the dimension 
	$d$ and the required accuracy. In the theorem below, we formalize this 
	intuition, where for simplicity we focus on approximating the indicator of 
	the unit ball:	
	\begin{theorem}[Approximability with 3-layer 
	networks]\label{thm:ellipsoidupbound}
		Given $ \delta>0 $, for any activation function $ \sigma $ satisfying 
		Assumption 1 in \citet{eldan2016power} and any continuous probability 
		distribution $ \mu 
		$ on $ \bbr^d $, there exists a constant $c_{\sigma}$ dependent only on $\sigma$, a constant $c_{\mu}$ dependent only on $\mu$ and a function $ g $ expressible by a 3-layer network of 
		width at most $ 
		\frac{2c_{\sigma}}{\sqrt{\delta}}\cdot \max\set{2c_{\mu}d^2,1} $, such that the following holds:
		\begin{equation*}
			\int_{\bbr^d}\p{g\p{\bx}-\ind{\norm{\bx}_2\le 1}}^2\mu\p{\bx}d\bx ~\le~ \delta.
		\end{equation*}	
	\end{theorem}
	The proof of the theorem appears in \subsecref{subsec:proofthml2ubound}.
	%\subsecref{subsec:proofthml2ubound}.
	
	%\note{Need to have theorem 1 and 3 in the same manner as the other	theorems, in terms of how the error is denoted (squared integral?)}
	
	\subsection{An Experiment}\label{sec:experiments}
		In this subsection, we empirically demonstrate that indicator functions 
		of balls are
		indeed easier to learn with a 3-layer network, compared to a 2-layer 
		network (even if the 2-layer network is significantly larger). This
		indicates that the depth/width trade-off for indicators of balls, 
		predicted 
		by our theory, can indeed be observed experimentally. Moreover, it	
		highlights the fact that our separation result is for simple natural 
		functions, that can be learned reasonably well from data using standard 
		methods.
		
		For our experiment, we sampled $ 5\cdot10^5 $ data instances in 
		$\reals^{100}$, with a direction chosen uniformly at random and a norm 
		drawn uniformly at random from the interval $ \pcc{0,2} $. To each 
		instance, we associated a target value computed according to the target 
		function $ f(\bx)=\ind{\norm{\bx}_2\le1} $. Another $5\cdot 10^4$ 
		examples were generated in a similar manner and used as a validation 
		set.
		
		We trained $5$ ReLU networks on this dataset: 
		\begin{itemize}
		\item One 3-layer network, with a first hidden layer of size $100$, 
		a second hidden layer of size $20$, and a linear output neuron. 
		\item Four 2-layer networks, with hidden layer of sizes $100,200,400$ 
		and $800$, and a linear output neuron. 
		\end{itemize}
		Training was performed with backpropagation, using the TensorFlow 
		library. We used the 
		squared loss $\ell(y,y')=(y-y')^2$ and batches of size  
		100. For all networks, we chose a momentum parameter of 0.95, and a
		learning rate starting at 0.1, decaying by a multiplicative factor of 
		0.95 every 1000 batches, and stopping at $ 10^{-4} $.
		%, and for 200000 SGD iterations which constitute 40 full data 
		%sweeps. 
		%all neurons have a bias term

		\begin{figure}[t]
				\centering
				\includegraphics[scale=0.4]{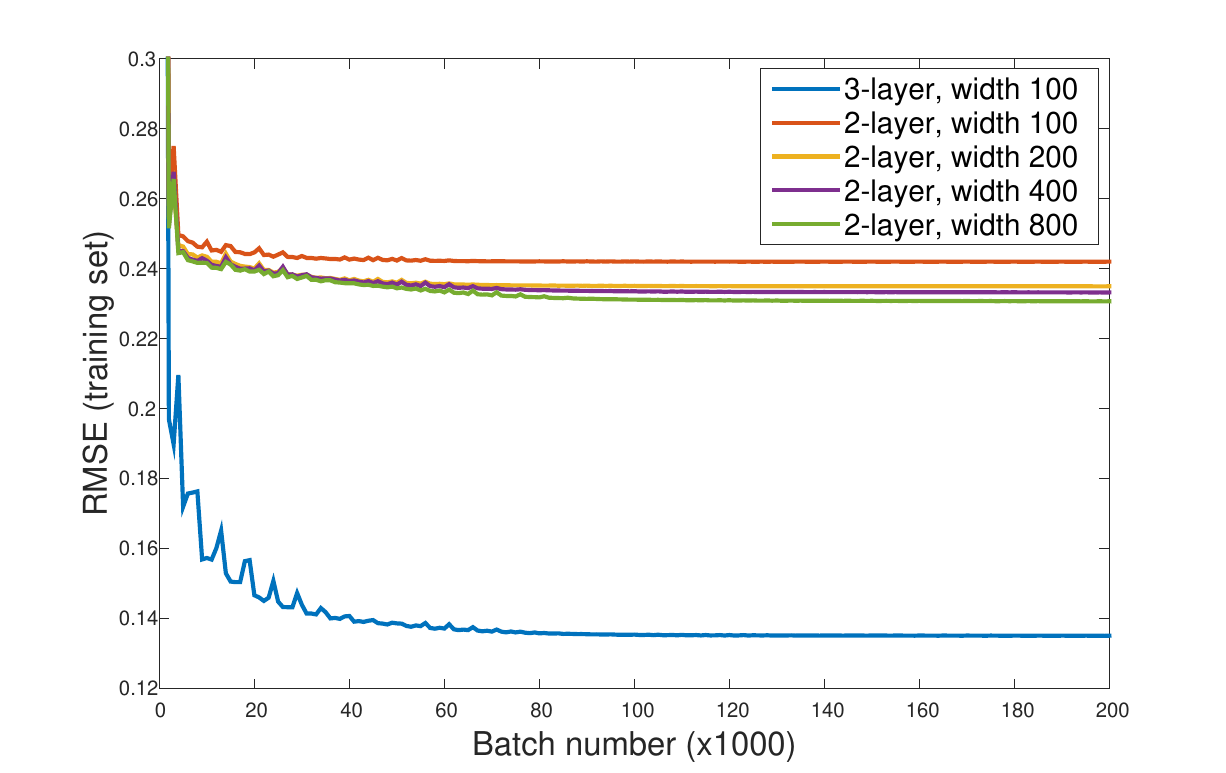}
				\includegraphics[scale=0.4]{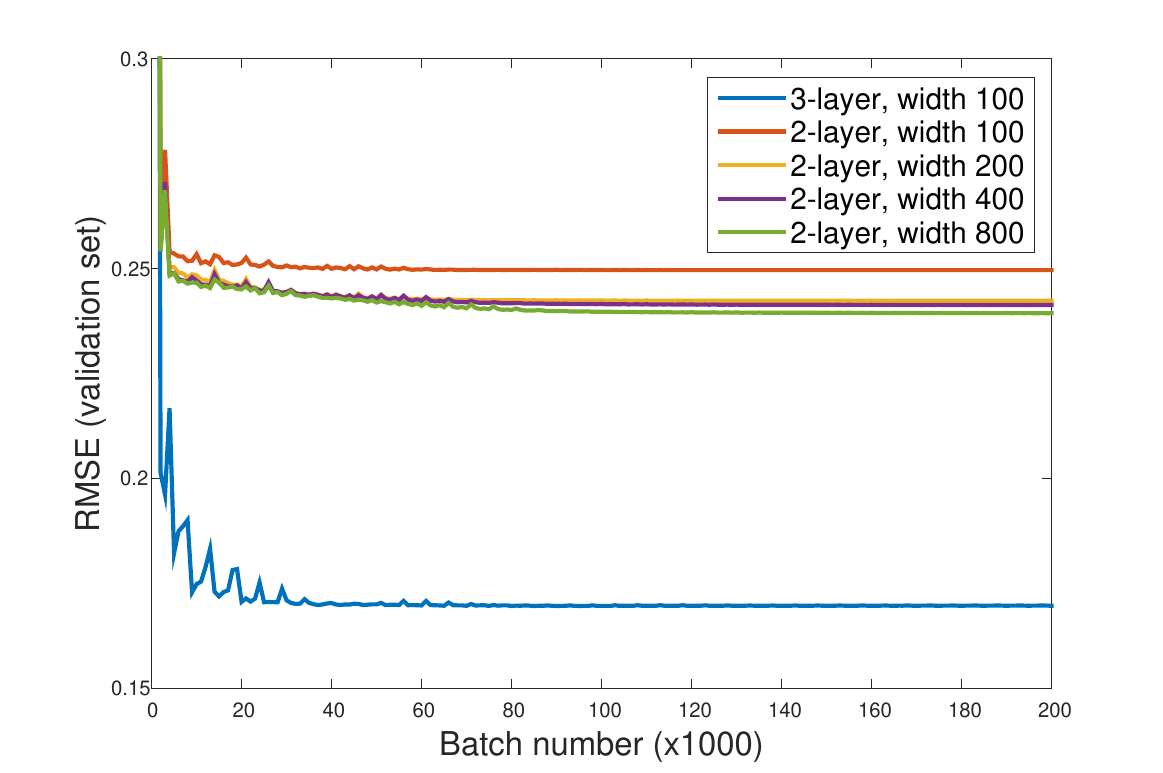}		
				\caption{\footnotesize The experiment results, depicting the 
				network's root mean square error over the training set (left) 
				and validation set (right), as a function of 
				the number of batches processed. Best viewed in color.}
				\label{fig:experiment}
		\par\end{figure}
		
		The results are presented in \figref{fig:experiment}. As can be clearly 
		seen, the 3-layer network achieves significantly better performance 
		than the 2-layer networks. This is true even though some of these 
		networks are significantly larger and with more parameters (for 
		example, the 2-layer, width 800 network has 
		\textasciitilde80K 
		parameters, vs. 
		\textasciitilde10K 
		parameters for the 3-layer network). This gap in performance is the 
		exact opposite of what might be expected based on parameter counting 
		alone. Moreover, increasing the width of the 2-layer networks exhibits 
		diminishing returns: The performance improvement in doubling the width 
		from 100 to 200 is much larger than doubling the width from 200 to 400 
		or 400 to 800. This indicates that one would need a much larger 
		2-layer network to match the 3-layer, width 100 network's performance. 
		Thus, we conclude that the network's depth indeed plays a crucial role, 
		and that 3-layer networks are inherently more suitable to express 
		indicator functions of the type we studied. 

		%The 3-layer network achieved loss 
		%	0.004, compared to the shallow networks of width 100, 200, 400, 
		%	800, obtaining losses 0.023, 0.17, 0.014, 0.012
		%	, respectively; corresponding to squared roots of 0.065 for the 
		%	3-layer network, and 0.152, 0.13, 0.118, 0.109 for the shallow 
		%	networks. These roughly convey the average distance in absolute 
		%	value of the target values from the network's prediction, allowing 
		%	us to better assess its performance \note{not 100\% OK with this 
		%	phrasing, got any other ideas?} 
			
		%\subsubsection{Experiment Discussion}
		%	We point out that the number of parameters on the trained networks 
		%	is $ 12141 $ for the 3-layer network, and $ 10201, 20401, 40801, 
		%	81601 $ for the shallow networks. 
		
	\section{$L_1$ Radial Functions; ReLU Networks}
	
	Having considered functions depending on the $L_2$ norm, we now turn to 
	consider functions depending on the $L_1$ norm. Focusing on ReLU networks, 
	we will show a certain separation result holding for
	\emph{any} non-linear function, which depends on the input $\bx$ only via 
	its 1-norm
	$\norm{\bx}_1$. 
	
	\begin{theorem}\label{thm:L1}
	Let $f:[0,\infty)\mapsto\reals$ be a function such that for some 
	$r,\delta>0$ 
	and
	$\epsilon\in (0,1/2)$, 
	\[
	\inf_{a,b\in \reals}\E_{x~\text{uniform on}~ 
	[r,(1+\epsilon)r]}[(f(x)-(ax-b))^2]> \delta~.
	\]
	Then there exists a distribution $ \gamma $ over $\{\bx:\norm{\bx}_1\leq 
	(1+\epsilon)r\}$, 
	such that if a 2-layer ReLU network $F(\bx)$ satisfies
	\[
	\int_{\bx}\left(f(\norm{\bx}_1)-F(\bx)\right)^2\gamma(\bx)d\bx ~\leq~ \delta/2,
	\]
	then its width must be at least
	$\tilde{\Omega}(\min\left\{1/\epsilon,\exp(\Omega(d))\right\})$ (where the 
	$\tilde{\Omega}$ notation hides 
	constants and factors logarithmic 
	in $\epsilon,d$).
	\end{theorem}
	To give a concrete example, suppose that $f(z)=[z-1]_+$, which cannot be 
	approximated by a linear function better than 
	$\Ocal(\epsilon^2)$ in an $\epsilon$-neighborhood of $1$. By taking 
	$r=1-\frac{\epsilon}{2}$ and $\delta=\Ocal(\epsilon^2)$, we get that no 
	2-layer network can approximate the function $[\norm{\bx}_1-1]_+$ (at 
	least with respect to some distribution), unless its width is 
	$\tilde{\Omega}(\min\left\{1/\epsilon,\exp(\Omega(d))\right\})$. On the 
	flip side, $f(\norm{\bx}_1)$ can be expressed \emph{exactly} by a 3-layer, 
	width $2d$ ReLU network: $\bx\mapsto 
	[\sum_{i=1}^{d}([x_i]_++[-x_i]_+)-1]_+$, where the output neuron is simply 
	the identity function. The same argument would work for any 
	piecewise-linear $f$. More generally, the same kind of argument would work 
	for any function $f$ exhibiting a non-linear behavior at some points: Such 
	functions can be well-approximated by 3-layer networks (by approximating 
	$f$ with a piecewise-linear function), yet any 
	approximating 2-layer network will have a lower bound on its size as 
	specified in the theorem. 
	
	Intuitively, the proof relies on showing that any good $2$-layer 
	approximation 
	of $f(\norm{\bx}_1)$ must capture the non-linear behavior of $f$ close to 
	``most'' points $\bx$ satisfying $\norm{\bx}_1\approx r$. However, a 
	$2$-layer 
	ReLU 
	network $\bx\mapsto \sum_{j=1}^{N}a_j\left[\inner{\bw_j,\bx}+b_j\right]_+$ 
	is 
	piecewise linear, with non-linearities only at the union of the $N$ 
	hyperplanes $\cup_{j}\{\bx:\inner{\bw_j,\bx}+b_j=0\}$. This implies that 
	``most'' points $\bx$ s.t. $\norm{\bx}_1\approx r$ must be $\epsilon$-close 
	to 
	a 
	hyperplane 
	 $\{\bx:\inner{\bw_j,\bx}+b_j=0\}$. However, the geometry of the $L_1$ ball 
	 $\{\bx:\norm{\bx}=r\}$ is 
	 such that the $\epsilon$ neighborhood of any single hyperplane can only 
	 cover 
	 a ``small'' portion of that ball, yet we need to cover most of the $L_1$ 
	 ball. Using this and an appropriate construction, we show that required 
	 number 
	 of hyperplanes is at least $1/\epsilon$, as long as $\epsilon > 
	 \exp(-\Ocal(d))$ 
	 (and if $\epsilon$ is smaller than that, we can simply use one 
	 neuron/hyperplane for each  of the $2^d$ facets of the $L_1$ ball, and get 
	 a 
	 covering using $2^d$ 
	 neurons/hyperplanes). The formal proof appears in \subsecref{subsec:proofthml1}.
	 %section \ref{subsec:proofthml1}. 
	 
	 We note that the bound in \thmref{thm:L1} is of a weaker nature than the 
	 bound in the previous section, in that the lower bound is only polynomial 
	 rather than exponential (albeit w.r.t. different problem parameters: 
	 $\epsilon$ vs. $d$). Nevertheless, we believe this does point out that 
	 $L_1$ balls also pose a geometric difficulty for 2-layer networks, and 
	 conjecture that our lower bound can be considerably improved: Indeed, at 
	 the moment we do not know how to approximate a function such as 
	 $\bx\mapsto [\norm{\bx}_1-1]_+$ with 2-layer networks to better than 
	 constant accuracy, using less than $\Omega(2^d)$ neurons.

	\section{$C^2$ Nonlinear Functions; ReLU Networks}\label{sec:C2}
	
	In this section, we establish a depth separation result for approximating 
	continuously twice-differentiable ($C^2$) functions using ReLU neural 
	networks. Unlike 
	the previous 
	results in this paper, the separation is for depths which can be larger 
	than 3, depending on the required approximation error. Also, the results 
	will all be with respect to the uniform distribution $\mu_d$ over 
	$[0,1]^d$. 
	As mentioned earlier, the results and techniques in this section 
	are closely related to the independent results of 
	\cite{yarotsky2016error,liang2016deep}, but our emphasis is on 
	$L_2$ rather than $L_{\infty}$ approximation bounds, and we focus on 
	somewhat different network architectures and function classes.
	
	Clearly, not all $C^2$ functions are difficult to approximate (e.g. a 
	linear function can be expressed exactly with a 2-layer network). Instead, 
	we consider functions which have a certain degree of non-linearity, in the 
	sense that its Hessians are non-zero along some direction, on a significant 
	portion of the domain. Formally, we make the following definition: 
	\begin{definition}\label{def:sigmalambda}
		Let $ \mu_d $ denote the uniform distribution on $ \pcc{0,1}^d $. For a function $ f:\pcc{0,1}^d\to\bbr $ and some $ \lambda>0 $, denote
		%\begin{equation*}
		%	\sigma_{\lambda}\p{f} = \sup_{\norm{\bv}_2=1, U:U\subseteq\pcc{0,1}^d\text{ connected and measurable,}\atop \bv^{\top}H\p{f}\p{\bx}\bv\ge\lambda ~\forall \bx\in U }\mu_d\p{U}
		%\end{equation*}
		%where $H(f)(\bx)$ is the Hessian of $f$ at $\bx$.
		\begin{equation*}
			\sigma_{\lambda}\p{f} = \sup_{\bv\in\mathbb{S}^{d-1},~U\in \mathcal{U}~\text{s.t.}~\bv^{\top}H(f)(\bx)\bv\ge\lambda~\forall \bx\in U }\mu_d\p{U},
		\end{equation*}
		where $ \mathbb{S}^{d-1}=\set{\bx:\norm{\bx}_2=1} $ is the 
		$d$-dimensional unit hypersphere, and $\mathcal{U}$ is the set of all 
		connected and measurable subsets of $\pcc{0,1}^d$.
	\end{definition}
	%\note{Make sure this new definition (which doesn't have $\bv$) is the one 
	%used in the proof}
		
	In words, $ \sigma_{\lambda}\p{f} $ is the measure (w.r.t. the uniform 
	distribution on $\pcc{0,1}^d$) of the largest connected set in the domain 
	of $ f $, where at any point, $ f $ has curvature at least $ \lambda $ 
	along some fixed direction $ \bv $. The ``prototypical'' functions $f$ we 
	are 
	interested in is when $\sigma_{\lambda}(f)$ is lower bounded by a constant 
	(e.g. it is $1$ if $f$ is strongly convex). We stress that our results in this section will hold equally well by considering the condition $ \bv^{\top}H(f)(\bx)\bv\le-\lambda $ as well, however for the sake of simplicity we focus on the former condition appearing in Def. \ref{def:sigmalambda}.	
	%\note{Maybe mention that our results could also holds equally well with the condition that $\bv^\top H(f)(\bx)\bv\leq -\lambda$, e.g. strongly	concave functions?} 
	Our goal is to show a depth separation result \emph{inidividually} for any such function (that is, for 
	any such function, there is a gap in the attainable error between deeper 
	and shallower networks, even if the shallow network is considerably 
	larger). 
	
	As usual, we start with an inapproximability result. Specifically, we prove 
	the following lower bound on the attainable approximation error of $f$, 
	using a ReLU neural network of a given depth and width: 

	\begin{theorem}\label{thm:c2_lowbound}
		For any $C^2$ function $ f:\pcc{0,1}^d\to\bbr $, any $\lambda>0$, and 
		any function $g$ on $[0,1]^d$ expressible by a ReLU network of depth 
		$l$ and maximal width $m$, it holds that 
		\begin{equation*}
		\int_{\pcc{0,1}^d} (f(\bx)-g(\bx)^2 \mu_d\p{\bx}d\bx \ge 
		\frac{c\cdot\lambda^2\cdot\sigma_\lambda^5}{(2m)^{4l}}~,
		\end{equation*}
		where $c>0$ is a universal constant.
	\end{theorem}
	The theorem conveys a key tradeoff between depth and 
	width when approximating a $C^2$ function using ReLU networks: The error 
	cannot decay faster than polynomially in the width $m$, yet the bound 
	deteriorates exponentially in the depth $l$. As we show later on, this 
	deterioration does not stem from the looseness in the bound: For 
	well-behaved $f$, it is indeed possible to construct ReLU networks, where 
	the approximation error decays exponentially with depth.
	
	The proof of \thmref{thm:c2_lowbound} appears in \subsecref{subsec:c2lbound}, and is based 
	on a series of intermediate results. First, we show that any 
	strictly curved function (in a 
	sense similar to Definition \ref{def:sigmalambda}) cannot be 
	well-approximated in an $L_2$ sense by piecewise linear functions, unless 
	the number of linear regions is large. 
	To that end, we first establish some necessary tools based on Legendre 
	polynomials. We then prove a result specific to the one-dimensional case, 
	including an explicit lower bound if the target function is quadratic 
	(\thmref{thm:quad_lbound}) or strongly convex or concave 
	(\thmref{thm:str_convex_lbound}). We then expand the construction 
	to get an error lower bound in general dimension $d$, depending on the 
	number of linear regions in the approximating piecewise-linear function. 
	Finally, we note that any ReLU network induces a piecewise-linear function, 
	and bound the number of linear regions induced by a ReLU network of 
	a given width and depth (using a lemma borrowed from 
	\cite{telgarsky2016benefits}). Combining this with the previous lower bound 
	yields \thmref{thm:c2_lowbound}. 
	
	We now turn to complement this lower bound with an approximability result, 
	showing that with more depth, a wide family of functions to which 
	\thmref{thm:c2_lowbound} applies \emph{can} be approximated with 
	exponentially high accuracy. Specifically, we consider functions which can 
	be approximated using a moderate number of multiplications and additions, 
	where the values of intermediate computations are bounded (for example, a 
	special case is any function approximable by a moderately-sized Boolean 
	circuit, or a polynomial).

	The key result to show this is the following, which implies that the 
	multiplication of two (bounded-size) numbers can be approximated by a 
	ReLU network, with error decaying exponentially with depth:
	%Recall the result in \corollaryref{cor:nets_lbound}, we have that the approximation error using a ReLU network is lower bounded by an expression which decays exponentially fast in the network`s depth, but only polynomially in the networks width. The reason for this is that depth allows for an exponentially efficient approximation of products, as demonstrated by the following theorem:
	\begin{theorem}\label{thm:product_approx}
		Let $ f:\pcc{-M,M}^2\to\bbr $, $ f\p{x,y} =x\cdot y$ and let $ \epsilon>0 $ be arbitrary. Then exists a ReLU neural network $ g $ of width $ 4\ceil{\log\p{\frac{M}{\epsilon}}}+13 $ and depth $ \ceil{2\log\p{\frac{M}{\epsilon}}}+9 $ satisfying
		\begin{equation*}
		\sup_{\p{x,y}\in\pcc{-M,M}^2}\abs{f\p{x,y} -g\p{x,y}}\le\epsilon.
		\end{equation*}
	\end{theorem}
	The idea of the construction is that depth allows us to compute highly-oscillating functions, which can extract high-order bits from the binary representation of the inputs. Given these bits, one can compute the product by a procedure resembling long multiplication, as shown in \figref{fig:quad_approx}, and formally proven as follows:
	\begin{figure}
		\centering
		\includegraphics[scale=0.4]{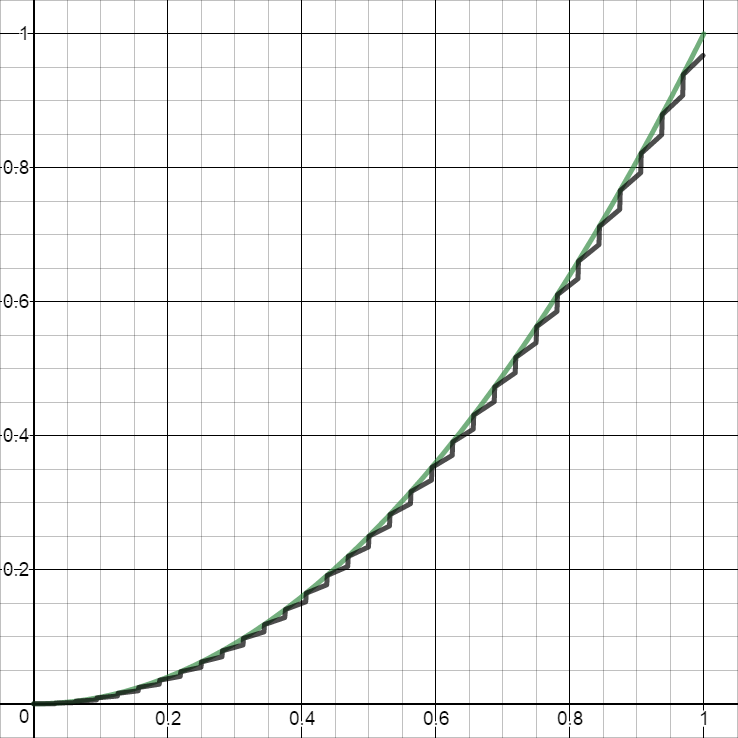}
		\caption{\footnotesize ReLU approximation of the function $ x\mapsto x^2 $ obtained by extracting 5 bits. The number of linear segments grows exponentially with the number of bits and the approximating network size.}
		\label{fig:quad_approx}
	\par\end{figure}
	
	\begin{proof}[Proof of \thmref{thm:product_approx}]
		We begin by observing that by using a simple linear change of variables on $ x $, we may assume without loss of generality that $ x\in\pcc{0,1} $, as we can just rescale $ x $ to the interval $ \pcc{0,1} $, and then map it back to its original domain $ \pcc{-M,M} $, where the error will multiply by a factor of $ 2M $. Then by requiring accuracy $ \frac{\epsilon}{2M} $ instead of $ \epsilon $, the result will follow.
		
		The key behind the proof is that performing bit-wise operations on the first $ k $ bits of $ x\in\pcc{0,1} $ yields an estimation of the product to accuracy $ 2^{1-k}M $. Let $ x=\sum_{i=1}^{\infty} 2^{-i}x_i$ be the binary representation of $ x $ where $ x_i $ is the $ i^\text{th} $ bit of $ x $, then
		\begin{align}\label{eq:binary_rep}
		x\cdot y &= \sum_{i=1}^{\infty}2^{-i}x_i\cdot y \nonumber\\
		&= \sum_{i=1}^{k}2^{-i}x_i\cdot y +\sum_{i=k+1}^{\infty}2^{-i}x_i\cdot y.
		\end{align}
		But since
		\begin{equation*}
		\abs{\sum_{i=k+1}^{\infty}2^{-i}x_i\cdot y} \le \abs{\sum_{i=k+1}^{\infty}2^{-i}\cdot y} = 2^{-k}\abs{y}\le2^{1-k}M,
		\end{equation*}
		\eqref{eq:binary_rep} implies
		\begin{equation*}
		\abs{x\cdot y-\sum_{i=1}^{k}2^{-i}x_i\cdot y}\le2^{1-k}M.
		\end{equation*}
		Requiring that $ 2^{2-k}M\le\frac{\epsilon}{2M} $, it suffices to show the existence of a network which approximates the function $ \sum_{i=1}^{k}2^{-i}x_i\cdot y $ to accuracy $ \frac{\epsilon}{2} $,
		where $ k=2\ceil{\log\p{\frac{8M}{\epsilon}}} $. This way both approximations will be at most $ \frac{\epsilon}{2} $, resulting in the desired accuracy of $ \epsilon $.
		
%		\note{maybe mention somewhere that our construction involves learning parity functions which are hard to learn? Otherwise this might suggest we're overselling}
		Before specifying the architecture which extracts the $ i^\text{th} $ bit of $ x $, we first describe the last 2 layers of the network. Let the penultimate layer comprise of $ k $ neurons, each receiving both $ y $ and $ x_i $ as input, and having the set of weights $ \p{2^{-i},1,-1} $. Thus, the output of the $ i^\text{th} $ neuron in the penultimate layer is $$ \relu{2^{-i}y+x_i-1}= 2^{-i}x_iy.$$ Let the final single output neuron have the set of weights $ \p{1,\dots,1,0}\in\bbr^{k+1} $, this way, the output of the network will be
		$ \sum_{i=1}^{k}2^{-i}x_i\cdot y $ as required.
		
		We now specify the architecture which extracts the first most significant $ k $ bits of $ x $. In \citet{telgarsky2016benefits}, the author demonstrates how the composition of the function
		\begin{equation*}
		\varphi\p{x}=\relu{2x}-\relu{4x-2}
		\end{equation*}
		with itself $ i $ times, $ \varphi^i $, yields a highly oscillatory triangle wave function in the domain $ \pcc{0,1} $. Furthermore, we observe that $ \varphi\p x=0~\forall x\le 0 $, and thus $ \varphi^i\p{x}=0 ~\forall x\le 0 $. Now, a linear shift of the input of $ \varphi^i $ by $ 2^{-i-1} $, and composing the output with
		\begin{equation*}
		\sigma_\delta\p{x} = \relu{\frac{1}{2\delta}x-\frac{1}{4\delta}+\frac{1}{2}} - \relu{\frac{1}{2\delta}x-\frac{1}{4\delta}-\frac{1}{2}},
		\end{equation*}
		which converges to $ x\mapsto\ind{x\ge0.5} $ as $ \delta\to0 $, results in an approximation of $ x\mapsto x_i $:
		\begin{equation*}
		\sigma_{\delta}\p{\varphi^i\p{x-2^{-i-1}}}.
		\end{equation*}
		We stress that choosing $ \delta $ such that the network approximates the bit-wise product $  $ to accuracy $ \frac{\epsilon}{2} $ will require $ \delta $ to be of magnitude $ \frac{1}{\epsilon} $, but this poses no problem as representing such a number requires $ \log\p{\frac{1}{\epsilon}} $ bits, which is also the magnitude of the size of the network, as suggested by the following analysis.
		
		Next, we compute the size of the network required to implement the above approximation. To compute $ \varphi $ only two neurons are required, therefore $ \varphi^i $ can be computed using $ i $ layers with $ 2 $ neurons in each, and finally composing this with $ \sigma_\delta $ requires a subsequent layer with $ 2 $ more neurons. To implement the $ i^\text{th} $ bit extractor we therefore require a network of size $ 2\times\p{i+1} $. Using dummy neurons to propagate the $ i^\text{th} $ bit for $ i<k $, the architecture extracting the $ k $ most significant bits of $ x $ will be of size $ 2k\times\p{k+1} $. Adding the final component performing the multiplication estimation will require $ 2 $ more layers of width $ k $ and $ 1 $ respectively, and an increase of the width by $ 1 $ to propagate $ y $ to the penultimate layer, resulting in a network of size $ \p{2k+1}\times\p{k+1} $.
	\end{proof}
	
	%Note that in this construction, depth is utilized in order to approximate the highly oscillatory $ x\mapsto x_i $, and indeed depth is necessary in general to obtain an exponential approximation of many functions as follows from \corollaryref{cor:nets_lbound}.
	
	\thmref{thm:product_approx} shows that multiplication can be performed 
	very accurately by deep networks. Moreover, additions can be computed by 
	ReLU networks exactly, using only a single layer with $ 4 $ neurons: Let $ 
	\alpha,\beta\in\bbr $ be arbitrary, then $ \p{x,y}\mapsto\alpha\cdot 
	x+\beta\cdot y $ is given in terms of ReLU summation by
	\begin{equation*}
	\alpha\relu{x}-\alpha\relu{-x}+\beta\relu{y}-\beta\relu{-y}.
	\end{equation*}
	Repeating these arguments, we see that any function which can be 
	approximated by a bounded number of operations involving additions and 
	multiplications, can also be approximated well by moderately-sized 
	networks. This is formalized in the following theorem, which provides an 
	approximation error upper bound (in the $L_{\infty}$ sense, which is 
	stronger than $L_2$ for upper bounds):
	%\note{Many readers will read corollary 3 before this part, so maybe it's 
	%best to define $ \mathcal{F} $ beforehand? Also, maybe quoting Poggio will 
	%make defining $ \mathcal{F} $ easier, which we should probably do 
	%anyway...}	
	\begin{theorem}\label{thm:main_ubound}
		Let $ \mathcal{F}_{t,M,\epsilon} $ be the family of functions on the 
		domain $ \pcc{0,1}^d $ with the property that $ 
		f\in\mathcal{F}_{t,M,\epsilon} $ is approximable to accuracy $ \epsilon 
		$ with respect to the infinity norm, using at most $ t $ operations 
		involving weighted addition, $ \p{x,y}\mapsto\alpha\cdot x+\beta\cdot y 
		$, where $ \alpha,\beta\in\bbr $ are fixed; and multiplication, $ 
		\p{x,y}\mapsto x\cdot y $, where each intermediate computation stage is 
		bounded in the interval $ \pcc{-M,M} $. Then there exists a universal 
		constant $ c $, and a ReLU network $ g $ of width and depth at most $ 
		c\p{t\log\p{\frac{1}{\epsilon}}+t^2\log\p{M}}$, such that
		\begin{equation*}
		\sup_{\bx\in\pcc{0,1}^d}\abs{f\p \bx-g\p \bx}\le2\epsilon.
		\end{equation*}
	\end{theorem}
	As discussed in \secref{sec:prelim}, this type of $L_{\infty}$ 
	approximation 
	bound implies an $L_{2}$ approximation bound with respect to any 
	distribution. The proof of the theorem appears in \subsecref{subsec:main_ubound_proof}.
	%\note{Make everything $\epsilon$ instead of $\epsilon/2$, and in final 
	%bound $2\epsilon$ rather than $\epsilon$}
	
	Combining \thmref{thm:c2_lowbound} and \thmref{thm:main_ubound}, we can 
	state the following corollary, which formally shows how depth can be 
	exponentially more valuable than width as a function of the target accuracy 
	$\epsilon$: 
	
	\begin{corollary}\label{cor:main_result}
		Suppose $ f\in 
		C^2\cap\mathcal{F}_{t\p{\epsilon},M\p{\epsilon},\epsilon} $, where $ 
		t\p{\epsilon}=\mathcal{O}\p{\textnormal{poly}\p{\log\p{1/\epsilon}}} $ 
		and 
		$ M\p{\epsilon}=\mathcal{O}\p{\textnormal{poly}\p{1/\epsilon}} $. Then 
		approximating $ f $ to accuracy $ \epsilon $ in the $ L_2 $ norm using 
		a fixed depth ReLU network requires width at least $ 
		\textnormal{poly}(1/\epsilon) $, whereas there exists a ReLU network of 
		depth and width at most $ p\p{\log\p{1/\epsilon}} $ which approximates 
		$ f $ to accuracy $ \epsilon $ in the infinity norm, where $ p $ is a 
		polynomial depending solely on $ f $.
	\end{corollary}	
	\begin{proof}
	
	%\note{Change what follows, here there should simply be an explanation/proof of how the corollary is derived.}
	
	The lower bound follows immediately from \thmref{thm:c2_lowbound}. For the 
	upper bound, observe that \thmref{thm:main_ubound} implies an $ \epsilon $ 
	approximation by a network of width and depth at most
	\begin{equation*}
	c\p{t\p{\epsilon/2}\log\p{2/\epsilon}+\p{t\p{\epsilon/2}}^2\log\p{M\p{\epsilon/2}}},
	\end{equation*}
	which by the assumption of \corollaryref{cor:main_result}, can be bounded by
	\begin{equation*}
	p\p{\log\p{1/\epsilon}}
	\end{equation*}
	for some polynomial $ p $ which depends solely on $ f $.
	\end{proof}
	
	%\newpage
	
	\subsection*{Acknowledgements}
	This research is supported in part by an FP7 Marie Curie CIG grant, Israel Science Foundation grant 425/13, and the Intel ICRI-CI Institute. We would like to thank Shai Shalev-Shwartz for some illuminating discussions, Eran Amar for his valuable help with the experiment, and Lv Xianzhang for spotting a mistake in a previous version of \secref{sec:indicators}.

	\bibliographystyle{icml2017}
	\bibliography{citations}

\begin{thebibliography}{14}
\providecommand{\natexlab}[1]{#1}
\providecommand{\url}[1]{\texttt{#1}}
\expandafter\ifx\csname urlstyle\endcsname\relax
  \providecommand{\doi}[1]{doi: #1}\else
  \providecommand{\doi}{doi: \begingroup \urlstyle{rm}\Url}\fi

\bibitem[Bianchini \& Scarselli(2014)Bianchini and
  Scarselli]{bianchinicomplexity}
Bianchini, M. and Scarselli, F.
\newblock On the complexity of shallow and deep neural network classifiers.
\newblock In \emph{ESANN}, 2014.

\bibitem[Cohen et~al.(2016)Cohen, Sharir, and Shashua]{cohen2016expressive}
Cohen, Nadav, Sharir, Or, and Shashua, Amnon.
\newblock On the expressive power of deep learning: A tensor analysis.
\newblock In \emph{29th Annual Conference on Learning Theory}, pp.\  698--728,
  2016.

\bibitem[Cybenko(1989)]{cybenko1989approximation}
Cybenko, George.
\newblock Approximation by superpositions of a sigmoidal function.
\newblock \emph{Mathematics of control, signals and systems}, 2\penalty0
  (4):\penalty0 303--314, 1989.

\bibitem[Delalleau \& Bengio(2011)Delalleau and Bengio]{delalleau2011shallow}
Delalleau, O. and Bengio, Y.
\newblock Shallow vs. deep sum-product networks.
\newblock In \emph{NIPS}, pp.\  666--674, 2011.

\bibitem[Eldan \& Shamir(2016)Eldan and Shamir]{eldan2016power}
Eldan, Ronen and Shamir, Ohad.
\newblock The power of depth for feedforward neural networks.
\newblock In \emph{29th Annual Conference on Learning Theory}, pp.\  907--940,
  2016.

\bibitem[He et~al.(2015)He, Zhang, Ren, and Sun]{he2015deep}
He, Kaiming, Zhang, Xiangyu, Ren, Shaoqing, and Sun, Jian.
\newblock Deep residual learning for image recognition.
\newblock \emph{arXiv preprint arXiv:1512.03385}, 2015.

\bibitem[Hornik(1991)]{hornik1991approximation}
Hornik, Kurt.
\newblock Approximation capabilities of multilayer feedforward networks.
\newblock \emph{Neural networks}, 4\penalty0 (2):\penalty0 251--257, 1991.

\bibitem[Liang \& Srikant(2016)Liang and Srikant]{liang2016deep}
Liang, Shiyu and Srikant, R.
\newblock Why deep neural networks?
\newblock \emph{arXiv preprint arXiv:1610.04161}, 2016.

\bibitem[Martens \& Medabalimi(2014)Martens and
  Medabalimi]{martens2014expressive}
Martens, J. and Medabalimi, V.
\newblock On the expressive efficiency of sum product networks.
\newblock \emph{arXiv preprint arXiv:1411.7717}, 2014.

\bibitem[Poggio et~al.(2016)Poggio, Mhaskar, Rosasco, Miranda, and
  Liao]{poggio2016and}
Poggio, Tomaso, Mhaskar, Hrushikesh, Rosasco, Lorenzo, Miranda, Brando, and
  Liao, Qianli.
\newblock Why and when can deep--but not shallow--networks avoid the curse of
  dimensionality: a review.
\newblock \emph{arXiv preprint arXiv:1611.00740}, 2016.

\bibitem[Poole et~al.(2016)Poole, Lahiri, Raghu, Sohl-Dickstein, and
  Ganguli]{poole2016exponential}
Poole, Ben, Lahiri, Subhaneil, Raghu, Maithreyi, Sohl-Dickstein, Jascha, and
  Ganguli, Surya.
\newblock Exponential expressivity in deep neural networks through transient
  chaos.
\newblock In \emph{Advances In Neural Information Processing Systems}, pp.\
  3360--3368, 2016.

\bibitem[Shaham et~al.(2016)Shaham, Cloninger, and Coifman]{shaham2016provable}
Shaham, Uri, Cloninger, Alexander, and Coifman, Ronald~R.
\newblock Provable approximation properties for deep neural networks.
\newblock \emph{Applied and Computational Harmonic Analysis}, 2016.

\bibitem[Telgarsky(2016)]{telgarsky2016benefits}
Telgarsky, Matus.
\newblock Benefits of depth in neural networks.
\newblock \emph{arXiv preprint arXiv:1602.04485}, 2016.

\bibitem[Yarotsky(2016)]{yarotsky2016error}
Yarotsky, Dmitry.
\newblock Error bounds for approximations with deep relu networks.
\newblock \emph{arXiv preprint arXiv:1610.01145}, 2016.

\end{thebibliography}

	\newpage
	
	\onecolumn

	\section{Proofs} \label{sec:proofs}

	\subsection{Proof of \thmref{thm:ellipsoidupbound}}\label{subsec:proofthml2ubound}
	
	This proof bears resemblance to the proof provided in 
	\citet{eldan2016power}[Lemma~10], albeit once approximating $ 
	\norm{\bx}_2^2 $, the following construction takes a slightly different 
	route. For completeness, we also state assumption 1 from 
	\citet{eldan2016power}:
	
	\begin{assumption}
	Given the activation function $\sigma$, there is a constant $c_\sigma\geq 
	1$ (depending only on $\sigma$) such that the following holds: For any 
	$L$-Lipschitz function $f:\reals\rightarrow\reals$ which is constant 
	outside a bounded interval $[-R,R]$, and for any $\delta$, there exist 
	scalars $a,\{\alpha_i,\beta_i,\gamma_i\}_{i=1}^{w}$, where $w\leq 
	c_\sigma\frac{RL}{\delta}$, such that the function
	\[
	h(x) = a+\sum_{i=1}^{w}\alpha_i\cdot\sigma(\beta_i x-\gamma_i)
	\]
	satisfies
	\[
	\sup_{x\in \reals} \left|f(x)-h(x)\right| \leq \delta.
	\]
	\end{assumption}	
	As discussed in \citet{eldan2016power}, this assumption is satisfied by 
	ReLU, sigmoid, threshold, and more generally all standard activation 
	functions we are familiar with.
	
	\begin{proof}
		Consider the $ 1 $-Lipschitz function
		\[
			l\p{x}=\min\set{x^2,1},
		\]
		which is constant outside $ \pcc{-1,1} $, as well as the function
		\[
			\ell\p{x}=\sum_{i=1}^dl(x_i)=\sum_{i=1}^d\min\set{x_i^2,1}
		\]
		on $ \bbr^d $. Applying assumption 1, we obtain a function $ 
		\tilde{l}(x) $ having the form $ a+\sum_{i=1}^w \alpha_i 
		\sigma\p{\beta_i x-\gamma_i} $ so that
		\[
			\sup_{x\in\bbr}\abs{\tilde{l}\p{x}-l\p{x}}\le\frac{\delta_1}{d},
		\]
		and where the width parameter $ w $ is at most $ \frac{c_{\sigma}d}{\delta_1} $. Consequently, the function
		\[
			\tilde{\ell}\p{\bx}=\sum_{i=1}^d\tilde{l}\p{x_i}
		\]
		can be expressed in the form $ a+\sum_{i=1}^w \alpha_i\sigma\p{\beta_i x-\gamma_i} $ where $ w\le\frac{c_{\sigma}d^2}{\delta_1} $, yielding an approximation satisfying
		\begin{equation}\label{eq:delta1}
			\sup_{\bx\in\bbr^d}\abs{\tilde{\ell}\p{\bx}-\ell\p{\bx}}\le\delta_1.
		\end{equation}
		We now invoke assumption 1 again to approximate the $ 1 $-Lipschitz function
		\begin{equation*}
			f\p{x}=\begin{cases}
				1 & x<0 \\
				1-x & x\in\pcc{0,1} \\
				0 & x>1
			\end{cases}
		\end{equation*}
		and obtain an approximation $ \tilde{f}\p{x}=\tilde{a}+\sum_{i=1}^{\tilde{w}} \tilde{\alpha}_i\sigma\p{\tilde{\beta}_i x-\tilde{\gamma}_i} $ satisfying
		\begin{equation}\label{eq:delta2}
			\sup_{x\in\bbr}\abs{\tilde{f}\p{x}-f\p{x}}\le \delta_2
		\end{equation}
		where $ \tilde{w}\le c_{\sigma}/2\delta_2 $.

		Now consider the composition $ g = \tilde{f}\circ \p{c_{\mu}\cdot\tilde{\ell}-c_{\mu}} $, where $ c_{\mu}>0 $ depends on $\mu$, to be determined later. This composition has the form
		\begin{equation*}
			a+\sum _{i=1}^wu_i\sigma\p{\sum_{j=1}^wv_{i,j}\sigma\p{\inner{\bw_{i,j},\bx}+b_{i,j}}+c_i}
		\end{equation*}
		for appropriate scalars $ a,u_i,c_i,v_{i,j},b_{i,j} $ and vectors $ \bw_{i,j} $, and where $ w $ is at most $ \max\set{c_{\sigma}d^2/\delta_1,c_{\sigma}/2\delta_2} $. It is now left to bound the approximation error obtained by $ g $.
		We have
		\begin{equation}\label{eq:tri_ineq}
			\norm{g-\ind{\norm{\bx}_2\le1}}_{L_2(\mu)} \le \norm{g-f\circ \p{c_{\mu}-c_{\mu}\cdot\ell}}_{L_2(\mu)} + \norm{f\circ \p{c_{\mu}-c_{\mu}\cdot\ell}-\ind{\norm{\bx}_2\le1}}_{L_2(\mu)}. 
		\end{equation}
		Beginning with the second summand, define for any $ \epsilon>0 $, $$ R_{\epsilon}=\set{\bx\in\bbr^d:1-\epsilon\le\norm{\bx}_2^2\le1}. $$ Since $ \mu $ is continuous, there exists $ \epsilon\in(0,1) $ such that
		\begin{equation}\label{eq:eps_ring}
			\int_{R_{\epsilon}}\mu\p{\bx}d\bx\le\frac{\sqrt{\delta}}{2}.
		\end{equation}
		We have for any $ \bx\in\bbr^d $ satisfying $ 1\le \norm{\bx}_2^2 $ that $\ell(\bx)=1$ and therefore 
		\begin{equation}\label{eq:outer}
			f\circ \p{c_{\mu}-c_{\mu}\cdot\ell}(\bx)=1.
		\end{equation} 
		Taking $ c_{\mu}=1/\epsilon $, we have for any $ \bx\in\bbr^d $ satisfying $ \norm{\bx}_2^2\le1-\epsilon $ that
		\begin{equation}\label{eq:inner}
			f\circ \p{c_{\mu}-c_{\mu}\cdot\ell}(\bx) \le f\p{\frac{1}{\epsilon} - \frac{1}{\epsilon}(1-\epsilon) } = f(1) = 0
		\end{equation}
		Combining both \eqref{eq:outer} and \eqref{eq:inner} we compute
		\begin{align}
			&\int_{\bbr^d}\p{f\p{c_{\mu}-c_{\mu}\cdot\ell\p{\bx}}-\ind{\norm{\bx}_2\le1}}^2\mu\p{\bx}d\bx \nonumber\\
			=& \int_{R_{\epsilon}}\p{f\p{c_{\mu}-c_{\mu}\cdot\ell\p{\bx}}-\ind{\norm{\bx}_2\le1}}^2\mu\p{\bx}d\bx \nonumber\\
			\le& \int_{R_{\epsilon}}1\cdot\mu(\bx)d\bx \le \frac{\sqrt{\delta}}{2}, \label{eq:cont_approx}
		\end{align}
		where the equality is since both functions are equal outside of $R_{\epsilon}$, the first inequality is since the difference between the two functions on $R_{\epsilon}$ is at most $1$, and the last inequality is due to \eqref{eq:eps_ring}. Moving to the first summand in \eqref{eq:tri_ineq}, we have
		\begin{align*}
			\norm{g-f\circ \p{c_{\mu}-c_{\mu}\cdot\ell}}_{L_2(\mu)} &\le \norm{g-f\circ \p{c_{\mu}-c_{\mu}\cdot\ell}}_{\infty} \\
			&\le \norm{g-f\circ \p{c_{\mu}-c_{\mu}\cdot\tilde{\ell}}}_{\infty} + \norm{f\circ \p{c_{\mu}-c_{\mu}\cdot\tilde{\ell}} - f\circ \p{c_{\mu}-c_{\mu}\cdot\ell}}_{\infty} \\
			&\le \delta_2 + \norm{\p{c_{\mu}-c_{\mu}\cdot\tilde{\ell}} - \p{c_{\mu}-c_{\mu}\cdot\ell}}_{\infty} \\
			&\le \delta_2 + c_{\mu}\norm{\tilde{\ell}-\ell}_{\infty} \le \delta_2 +c_{\mu}\delta_1,
		\end{align*}
		where the third inequality is due to \eqref{eq:delta2} and $f$ being $1$-Lipschitz and the last inequality due to \eqref{eq:delta1}. Letting $\delta_2=\sqrt{\delta}/4$
		and $\delta_1=\sqrt{\delta}/4c_{\mu}$ which also entails $ w\le \frac{2c_{\sigma}}{\sqrt{\delta}}\cdot \max\set{2c_{\mu}d^2,1} $, the above is upper bounded by $\sqrt{\delta}/2$, which when combined with \eqref{eq:cont_approx} and plugged in \eqref{eq:tri_ineq} implies the lemma.
	\end{proof}
		
	\subsection{Proof of \thmref{thm:L1}}\label{subsec:proofthml1}
	
	Consider an input distribution of the 
	form
	\[
	\bx = sr\bv,
	\]
	where $\bv$ is drawn from a certain distribution on the unit $L_1$ sphere
	$\{\bx:\norm{\bx}_1=1\}$ to be specified later, and $s$ is uniformly 
	distributed on $[1,1+\epsilon]$.
	
	Let 
	\[
	F(\bx) = \sum_{j=1}^{N}a_j\left[\inner{\bw_j,\bx}+b_j\right]_+
	\]
	be a 2-layer ReLU network of width $N$, such that with respect to the 
	distribution above, 
	\[
	\E_{\bx}\left[(f(\norm{\bx}_1)-F(\bx))^2\right] ~=~ 
	\E_{\bv}\left[\E_{s}\left[(f(sr)-F(sr\bv))^2\middle|\bv\right]
	\right]~\leq~
	\frac{\delta_{\epsilon}}{2}~.
	\]
	By Markov's inequality, this implies
	\[
	{\Pr}_{\bv}\left(\E_{s}\left[f(sr)-F(sr\bv)^2\middle|
	\bv\right]\leq \delta_{\epsilon}\right)~\geq~
	 \frac{1}{2}~.
	\]
	By the assumption on $f$, and the fact that $s$ is uniform on 
	$[1,1+\epsilon]$, 
	we have that
	$\E_{s}\left[f(sr)-F(sr\bv)^2\middle|\bv\right]\leq
	 \delta_{\epsilon}$ only if $\tilde{f}_N$ is not a linear function on the 
	 line 
	 between 
	 $r\bv$ and $(1+\epsilon)r\bv$. In other words, this line must be 
	 crossed by the hyperplane $\{\bx:\inner{\bw_j,\bx}+b_j=0\}$ for some 
	 neuron 
	 $j$. Thus, we must have
	\begin{equation}\label{eq:prlowhalf}
	{\Pr}_{\bv}\left(\exists j\in \{1,\ldots,N\},s\in 
	[1,1+\epsilon]~~\text{s.t.}~~\inner{\bw_j,sr\bv}+b_j=0\right)~\geq~\frac{1}{2}.
	\end{equation}
	The left hand side equals
	\begin{align*}
	&{\Pr}_{\bv}\left(\exists j\in \{1\ldots N\},s\in 
	[1,1+\epsilon]~\text{s.t.} \inner{\bw_j,
	\bv}=-b_j/sr\right)\\
	&=
	{\Pr}_{\bv}\left(\exists j\in \{1\ldots N\}
	~\text{s.t.}~\inner{\bw_j,
	\bv}~\text{between}~-\frac{b_j}{r}~\text{and}~-\frac{b_j}{(1+\epsilon)r}\right)\\
	&\leq
	{\Pr}_{\bv}\left(\exists j\in \{1 \ldots N\}
	~\text{s.t.} \inner{\bw_j,
	\bv}~\text{betwen}~-\frac{b_j}{r}~\text{and}~-(1-\epsilon)\frac{b_j}{r}\right)\\
	&\leq \sum_{j=1}^{N}{\Pr}_{\bv}\left(\inner{\bw_j,
	\bv}~\text{between}~-\frac{b_j}{r}~\text{and}~-(1-\epsilon)\frac{b_j}{r} 
	\right)\\
	&\leq N\cdot 
	\sup_{\bw\in\reals^d,b\in\reals}{\Pr}_{\bv}\left(\inner{\bw,\bv}~\text{between}~
	-\frac{b}{r}~\text{and}~-(1-\epsilon)\frac{b}{r}\right)\\
	&=N\cdot \sup_{\bw\in\reals^d,b\in \reals}{\Pr}_{\bv}\left(\left\langle 
	\frac{-r\bw}{b},\bv\right\rangle\in 
	[1-\epsilon,1]\right)~=~N\cdot 
	\sup_{\bw\in\reals^d}{\Pr}_{\bv}\left(\inner{\bw,\bv}\in 
	[1-\epsilon,1]\right),
	\end{align*}
	where in the first inequality we used the fact that 
	$\frac{1}{1+\epsilon}\geq 
	1-\epsilon$ for all $\epsilon \in (0,1)$, and in the second inequality we 
	used 
	a union bound. Combining these inequalities with 
	\eqref{eq:prlowhalf}, we get that
	\[
	N~\geq~ \frac{1}{\sup_{\bw}{\Pr}_{\bv}\left(\inner{\bw,\bv}\in 
	[1-\epsilon,1]\right)}~.
	\]
	As a result, to prove the theorem, it is enough to construct a distribution 
	for 
	$\bv$ on the on the unit $L_1$ ball, such that for any $\bw$, 
	\begin{equation}\label{eq:needtoshow}
	{\Pr}_{\bv}\left(\inner{\bw,\bv}\in 
	[1-\epsilon,1]\right)~\leq~ \tilde{\Ocal}(\epsilon+\exp(-\Omega(d)))
	\end{equation}
	By the inequality above, we would 
	then get that $N=\tilde{\Omega}(\min\{1/\epsilon,\exp(\Omega(d))\})$. 
	
	Specifically, consider a distribution over $\bv$ defined as follows: First, 
	we 
	sample 
	$\bsigma\in 
	\{-1,+1\}^d$ uniformly at random, and $\bn\in\reals^d$ from a standard 
	Gaussian 
	distribution, and define
	\[
	\hat{\bv} = 
	\frac{1}{d}\left(\bsigma+c_d 
	\left(I-\frac{1}{d}\bsigma\bsigma^\top\right)\bn\right),
	\]
	where $c_d>0$ is a parameter dependent on $d$ to be determined later. It is 
	easily verified that 
	$\inner{\bsigma/d,\hat{\bv}}=\inner{\bsigma/d,\bsigma/d}$ 
	independent of $\bn$, hence $\hat{\bv}$ lies on the hyperplane containing 
	the 
	facet of the $L_1$ ball on which $\bsigma/d$ resides. Calling this facet 
	$F_{\bsigma}$, we define $\bv$ to have the same distribution as 
	$\hat{\bv}$, 
	conditioned on $\hat{\bv}\in F_{\bsigma}$. 
	
	We begin by arguing that
	\begin{equation}\label{eq:needtoshow2}
	{\Pr}_{\bv}\left(\inner{\bw,\bv}\in 
	[1-\epsilon,1]\right) ~\leq~ 2\cdot{\Pr}_{\hat{\bv}} 
	\left(\inner{\bw,\hat{\bv}}\in 
	[1-\epsilon,1]\right).
	\end{equation}
	To see this, let $A=\{\bx:\inner{\bw,\bx}\in [1-\epsilon,1]\}$, and note 
	that 
	the left hand side equals
	\begin{align*}
	\Pr(\bv\in A) &= \E_{\bsigma}[\Pr(\bv\in A|\bsigma)] = 
	\E_{\bsigma}\left[\Pr\left(\hat{\bv}\in A|\bsigma,\hat{\bv}\in 
	F_{\bsigma}\right)\right]\\
	&= \E_{\bsigma}\left[\frac{\Pr(\hat{\bv}\in A\cap 
	F_{\bsigma}|\bsigma)}{\Pr(\hat{\bv}\in F_{\bsigma}|\bsigma)}\right]
	~\leq~ \frac{1}{\min_{\bsigma}\Pr(\hat{\bv}\in 
	F_{\bsigma}|\bsigma)}\E_{\bsigma}\left[\Pr(\hat{\bv}\in A|\bsigma)\right]\\
	&= \frac{\Pr(\hat{\bv}\in A)}{\min_{\bsigma}\Pr(\hat{\bv}\in 
	F_{\bsigma}|\bsigma)}.
	\end{align*}
	Therefore, to prove \eqref{eq:needtoshow2}, it is enough to prove that 
	$\Pr(\hat{\bv}\in F_{\bsigma}|\bsigma)\geq 1/2$ for any $\bsigma$. As shown
	earlier, $\hat{\bv}$ lies on the hyperplane containing $F_{\bsigma}$,
	the facet of the $L_1$ ball in which $\bsigma/d$ resides. Thus, $\hat{\bv}$ 
	can 
	be outside $F_{\bsigma}$, only if at least one of its coordinates has a 
	different sign than $\bsigma$. By definition of $\hat{\bv}$, this can only 
	happen if $\norm{c_d(I-\bsigma\bsigma^\top/d)\bn}_{\infty}\geq 1$. The 
	probability of this event (over the random draw of $\bn$) equals
	\[
	\Pr\left(\max_{j\in\{1\ldots d\} }
	\left|c_d\left(n_j-\frac{1}{d}\inner{\bsigma,\bn}\sigma_j\right)\right|\geq 
	1\right)
	~=~
	\Pr\left(\max_{j\in\{1\ldots d\} } \left|n_j-\sigma_j\cdot 
	\frac{1}{d}\sum_{i=1}^{d}\sigma_i 
	n_i\right|\geq \frac{1}{c_d}\right).
	\]
	Since $\sigma_i\in \{-1,1\}$ for all $i$, the event on the right hand side 
	can 
	only occur if $|n_j|\geq 1/2c_d$ for some 
	$j$. Recalling that each $n_j$ has a standard Gaussian distribution, this 
	probability can be upper bounded by
	\[
	\Pr\left(\max_{j\in\{1\ldots d\} } |n_j|\geq \frac{1}{2c_d}\right)~\leq~ 
	d\cdot\Pr\left(|n_1|\geq 
	\frac{1}{2c_d}\right)~=~2d\cdot\Pr\left(n_1\geq 
	\frac{1}{2c_d}\right)~\leq~2d\cdot\exp\left(-\frac{1}{4c_d^2}\right),
	\]
	where we used a union bound and a standard Gaussian tail bound. Thus, by 
	picking 
	\[
	c_d = \sqrt{\frac{1}{4\log(4d)}},
	\]
	we can ensure that the probability is at most $1/2$, hence proving that 
	$\Pr(\hat{\bv}\in F_{\bsigma}|\bsigma)\geq 1/2$ and validating 
	\eqref{eq:needtoshow2}. 
	
	With \eqref{eq:needtoshow2} in hand, we now turn to upper bound
	\[
	{\Pr}\left(\inner{\bw,\hat{\bv}}\in 
	[1-\epsilon,1]\right)~=~\Pr\left(\frac{1}{d}\left\langle\bw,\bsigma+c_d 
	\left(I-\frac{1}{d}\bsigma\bsigma^\top\right)\bn\right\rangle\in 
	[1-\epsilon,1]\right).
	\]
	By the equation above, we have that conditioned on $\bsigma$, the 
	distribution 
	of $\inner{\bw,\hat{\bv}}$ is Gaussian with mean $\inner{\bsigma,\bw}/d$ 
	and 
	variance
	\[
	\frac{c_d^2}{d^2}\cdot 
	\bw^\top\left(I-\frac{1}{d}\bsigma\bsigma^\top\right)^2\bw
	 ~=~ 
	 \frac{c_d^2}{d^2}\cdot\left(\norm{\bw}^2-\frac{\inner{\bw,\bsigma}^2}{d}\right)
	 ~=~ \left(\frac{c_d \norm{\bw}}{d}\right)^2\cdot 
	 \left(1-\frac{1}{d}\left\langle
	 \frac{\bw}{\norm{\bw}},\bsigma\right\rangle^2\right).
	\]
	By Hoeffding's inequality, we have that for any $t>0$,
	\[
	{\Pr}_{\bsigma}\left(\left|\frac{\inner{\bsigma,\bw}}{d}\right|>t
	\cdot\frac{\norm{\bw}}{d}\right)~\leq~
	2\exp(-2t^2)
	\]
	and
	\[
	{\Pr}_{\bsigma}\left(\left|
	 \left\langle\frac{\bw}{\norm{\bw}},\bsigma\right\rangle\right| > 
	 \sqrt{\frac{d}{2}}\right)~\leq 2\exp(-d).
	\]
	This means that with probability at least $1-2\exp(-d)-2\exp(-2t^2)$ over 
	the 
	choice of $\bsigma$, the distribution of $\inner{\bw,\hat{\bv}}$ 
	(conditioned 
	on $\bsigma$) is Gaussian with mean bounded in absolute value by 
	$t\norm{\bw}/d$, and variance of at least 
	$\left(\frac{c_d \norm{\bw}}{d}\right)^2\cdot \left(1-\frac{1}{d}\cdot 
	\frac{d}{2}\right)= \frac{1}{2}\left(\frac{c_d\norm{\bw}}{d}\right)^2$. To 
	continue, we utilize the following lemma:
	
	\begin{lemma}
	Let $n$ be a Gaussian random variable on $\reals$ with mean $\mu$ and 
	variance $v^2$ for some $v>0$. Then for any $\epsilon\in (0,1)$, 
	\[
	\Pr\left(n\in [1-\epsilon,1]\right)~\leq~ 
	\sqrt{\frac{2}{\pi}}\cdot 
	\max\left\{1,\frac{|\mu|}{v}\right\}\cdot\frac{\epsilon}{1-\epsilon}.
	\]
	\end{lemma}
	\begin{proof}
	Since the probability can only increase if we replace the mean $\mu$ by 
	$|\mu|$, we will assume without loss of generality that $\mu\geq 0$. 
	
	By definition of a Gaussian distribution, and using the easily-verified 
	fact 
	that $\exp(-z^2)\leq \min\{1,1/z\}$ for all $z\geq 0$, the probability 
	equals
	\begin{align*}
	\frac{1}{\sqrt{2\pi 
	v^2}}&\int_{1-\epsilon}^{1}\exp\left(-\frac{(x-\mu)^2}{v^2}\right)dx
	~\leq~ \frac{\epsilon}{\sqrt{2\pi v^2}}\cdot \max_{x\in 
	[1-\epsilon,1]}\exp\left(-\frac{(x-\mu)^2}{v^2}\right)\\
	&~\leq~ \frac{\epsilon}{\sqrt{2\pi v^2}}\cdot\max_{x\in 
	[1-\epsilon,1]}\min\left\{1,\frac{v}{|x-\mu|}\right\}~=~
	\frac{\epsilon}{\sqrt{2\pi}}\cdot\max_{x\in 
	[1-\epsilon,1]}\min\left\{\frac{1}{v},\frac{1}{|x-\mu|}\right\}\\
	&~=~\frac{\epsilon}{\sqrt{2\pi}}\cdot\max_{x\in 
	[1-\epsilon,1]}\frac{1}{\max\{v,|x-\mu|\}}~=~
	\frac{\epsilon}{\sqrt{2\pi}}\cdot\max_{x\in 
	[1-\epsilon,1]}\frac{\max\{1,\frac{\mu}{v}\}}{\max\{1,\frac{\mu}{v}\}\cdot
	\max\{v,|x-\mu|\}}\\
	&~\leq~
	\frac{\epsilon}{\sqrt{2\pi}}\cdot\max_{x\in 
	[1-\epsilon,1]}\frac{\max\{1,\frac{\mu}{v}\}}{\max\{\mu,|x-\mu|\}}
	~=~
	\frac{\epsilon}{\sqrt{2\pi}}\cdot\frac{\max\{1,\frac{\mu}{v}\}}
	{\max\{\mu,\min_{x\in[1-\epsilon,1]}|x-\mu|\}}
	\end{align*}
	A simple case analysis reveals that the denominator is at 
	least\footnote{If $\mu\in [1-\epsilon,1]$, then we get 
	$\max\{\mu,0\}=\mu\geq 
	1-\epsilon$. If $\mu>1$, we get $\max\{\mu,\mu-1\}>1\geq 1-\epsilon$.
	If $\mu<1-\epsilon$, we get $\max\{\mu,1-\epsilon-\mu\}\geq 
	(1-\epsilon)/2$.} $\frac{1-\epsilon}{2}$, from which the result follows.
	\end{proof}
	Using this lemma and the previous observations, we get that with 
	probability at 
	least 
	$1-2\exp(-d)-2\exp(-2t^2)$ over the choice of $\bsigma$,
	\begin{align*}
	\Pr(\inner{\bw,\hat{\bv}}\in [1-\epsilon,1]|\bsigma) &~\leq~ 
	\sqrt{\frac{2}{\pi}}\cdot 
	\max\left\{1,\frac{t\norm{\bw}/d}{c_d\norm{\bw}/\sqrt{2} d}\right\}\cdot 
	\frac{\epsilon}{1-\epsilon}\\
	&~=~ \sqrt{\frac{2}{\pi}}\cdot 
	\max\left\{1,\frac{t}{c_d\sqrt{2}}\right\}\cdot 
	\frac{\epsilon}{1-\epsilon}.
	\end{align*}
	Letting $E$ be the event that $\bsigma$ is such that this inequality is 
	satisfied (and noting that its probability of non-occurence is at most 
	$2\exp(-d)+2\exp(-2t^2)$), we get overall that
	\begin{align*}
	\Pr(\inner{\bw,\hat{\bv}}\in [1-\epsilon,1]) &~=~
	\Pr(E)\cdot\Pr(\inner{\bw,\hat{\bv}}\in [1-\epsilon,1]|E)+\Pr(\neg 
	E)\cdot\Pr(\inner{\bw,\hat{\bv}}\in [1-\epsilon,1]|\neg E)\\
	&~\leq~
	1\cdot \Pr(\inner{\bw,\hat{\bv}}\in [1-\epsilon,1]|E)+\Pr(\neg E)\cdot 1\\
	&~\leq~
	\sqrt{\frac{2}{\pi}}\cdot \max\left\{1,\frac{t}{c_d\sqrt{2}}\right\}\cdot 
	\frac{\epsilon}{1-\epsilon}+2\exp(-d)+2\exp(-2t^2).
	\end{align*}
	Recalling \eqref{eq:needtoshow2} and the definition of $c_d$, we get that
	\[
	\Pr(\inner{\bw,\bv}\in [1-\epsilon,1])~\leq~ 
	\sqrt{\frac{8}{\pi}}\cdot 
	\max\left\{1,t\cdot\sqrt{2\log(4d)}\right\}\cdot 
	\frac{\epsilon}{1-\epsilon}+2\exp(-d)+2\exp(-2t^2).
	\]
	Picking $t = 
	\sqrt{\frac{1}{2}\log\left(\frac{1-\epsilon}{\epsilon}\right)}$, 
	we get the bound
	\[
	\left(\sqrt{\frac{8}{\pi}}\cdot 
	\max\left\{1,\sqrt{\log\left(\frac{1-\epsilon}{\epsilon}\right)\log(4d)}\right\}+2\right)\cdot
	\frac{\epsilon}{1-\epsilon}+2\exp(-d)~=~ 
	\tilde{\Ocal}\left(\epsilon+\exp(-d)\right).
	\]
	This justifies \eqref{eq:needtoshow}, from which the result follows.
	
	\subsection{Proof of \thmref{thm:c2_lowbound}} \label{subsec:c2lbound}
	
	The proof rests largely on the following key result: 
	\begin{theorem}\label{thm:main_lbound}
		Let $ \mathcal{G}_n $ be the family of piece-wise linear functions on 
		the domain $ \pcc{0,1} $ comprised of at most $ n $ linear segments. 
		Let $ \mathcal{G}_n^d $ be the family of piece-wise linear functions on 
		the domain $ \pcc{0,1}^d $, with the property that for any $ 
		g\in\mathcal{G}_n^d $ and any affine transformation $ h:\bbr\to\bbr^d 
		$, $ g\circ h \in \mathcal{G}_n $. Suppose $ f:\pcc{0,1}^d\to\bbr $ is 
		$ C^2 $. Then for all $ \lambda>0 $
		\begin{equation*}
		\inf_{g\in\mathcal{G}_n^d}\int_{\pcc{0,1}^d} (f-g)^2 d\mu_d \ge 
		\frac{c\cdot\lambda^2\cdot\sigma_\lambda(f)^5}{n^4},
		\end{equation*}
		where $ c=\frac{5}{4096} $.
	\end{theorem}
	
	\thmref{thm:main_lbound} establishes that the error of a piece-wise linear 
	approximation of a $ C^2 $ function cannot decay faster than quartically in 
	the number of linear segments of any \emph{one-dimensional} projection of 
	the approximating function. Note that this result is stronger than a bound 
	in terms of the total number of linear regions in $\bbr^d$, since that 
	number can be exponentially higher (in the dimension) than $n$ as defined 
	in the theorem.
	
	Before proving \thmref{thm:main_lbound}, let us explain how we can use it 
	to 
	prove \thmref{thm:c2_lowbound}. To that end, we use the result in 
	\citet[Lemma~3.2]{telgarsky2016benefits}, of which the following 
	is an immediate corollary:
	\begin{corollary}\label{cor:telgarsky_ubound}
		Let $ \mathcal{N}_{m,l}^d $ denote the family of ReLU neural networks 
		receiving input of dimension $ d $ and having depth $ l $ and maximal 
		width $ m $. Then
		\begin{equation*}
		\mathcal{N}_{m,l}^d\subseteq \mathcal{G}_{\p{2m}^l}^d.
		\end{equation*}
	\end{corollary}
	Combining this corollary with \thmref{thm:main_lbound}, the result 
	follows. The remainder of this subsection will be devoted to proving 
		\thmref{thm:main_lbound}.
	
%	
%	The following corollary bears great resemblance to the lower bound provided 
%	in \citet{telgarsky2016benefits}. Albeit lower bounding the accuracy to 
%	which a ReLU network of a given size can approximate a certain function, 
%	rather than indicating what is a lower bound on the minimal size required 
%	for achieving non-constant approximation error. By combining 
%	\thmref{thm:main_lbound} and \corollaryref{cor:telgarsky_ubound}, we have 
%	that
%	\begin{corollary}\label{cor:nets_lbound}
%		Suppose $ f:\pcc{0,1}^d\to\bbr $ is $ C^2 $. Then for all $ \lambda>0 $,
%		\begin{equation*}
%		\inf_{g\in\mathcal{N}_{m,l}^d}\int_{\pcc{0,1}^d} (f-g)^2 d\mu_d \ge 
%		\frac{c\cdot\lambda^2\cdot\sigma_\lambda^5}{2^lm^l}.
%		\end{equation*}
%	\end{corollary}

	\subsubsection{Some Technical Tools}
	\begin{definition}
		Let $ P_i $ denote the $ i^\text{th} $ Legendre Polynomial given by Rodrigues' formula:
		\begin{equation*}
		P_i\p{x}=\frac{1}{2^i i!}\frac{d^i}{dx^i}\pcc{\p{x^2-1}^i}.
		\end{equation*}
	\end{definition}
	These polynomials are useful for the following analysis since they obey the orthogonality relationship
	\begin{equation*}
	\int_{-1}^1 P_i\p{x}P_j\p{x}dx = \frac{2}{2i+1}\delta_{ij}.
	\end{equation*}
	Since we are interested in approximations on small intervals where the approximating function is linear, we use the change of variables $ x = \frac{2}{\ell}t-\frac{2}{\ell}a-1 $ to obtain an orthogonal family $ \set{\tilde{P}_i}_{i=1}^\infty $ of shifted Legendre polynomials on the interval $ \pcc{a,a+\ell} $ with respect to the $ L_2 $ norm. The first few polynomials of this family are given by
	\begin{align}
	\tilde{P}_0\p{x}&=1	\nonumber\\
	\tilde{P}_1\p{x}&=\frac{2}{\ell}x-\p{\frac{2}{\ell}a+1} \nonumber\\
	\tilde{P}_2\p{x}&=\frac{6}{\ell^2}x^2-\p{\frac{12a}{\ell^2}+\frac{6}{\ell}}x+\p{\frac{6a^2}{\ell^2}+\frac{6a}{\ell}+1}. \label{eq:P2}
	\end{align}
	The shifted Legendre polynomials obey the orthogonality relationship
	\begin{equation}\label{eq:shift_ortho}
	\int_a^{a+\ell} \tilde{P}_i\p{x}\tilde{P}_j\p{x}dx = \frac{\ell}{2i+1}\delta_{ij}.
	\end{equation}
	\begin{definition}
		We define the Fourier-Legendre series of a function $ f:\pcc{a,a+\ell}\to \bbr $ to be
		\begin{equation*}
		f\p{x}=\sum_{i=0}^{\infty}\tilde{a}_i\tilde{P}_i\p{x},
		\end{equation*}
		where the Fourier-Legendre Coefficients $ \tilde{a}_i $ are given by
		\begin{equation*}\label{eq:leg_coefficients}
		\tilde{a}_i=\frac{2i+1}{\ell}\int_a^{a+\ell} \tilde{P}_i\p{x}f\p{x}dx.
		\end{equation*}
	\end{definition}
	
	\begin{theorem}
		A generalization of Parseval's identity yields
		\begin{equation}\label{eq:gen_parseval}
		\norm{f}_{L_2}^2 = \ell\sum_{i=0}^{\infty}\frac{\tilde{a}_i^2}{2i+1}.
		\end{equation}
	\end{theorem}
	
	\begin{definition}
		A function $f$ is \emph{$\lambda$-strongly convex} if for all $\bw,\bu$ and $\alpha\in (0,1)$,
		\[
		f(\alpha\bw+(1-\alpha)\bu)\leq \alpha f(\bw)+(1-\alpha)f(\bu)-\frac{\lambda}{2}\alpha(1-\alpha)\norm{\bw-\bu}_2^2.
		\]
		A function is \emph{$\lambda$-strongly concave}, if $-f$ is $\lambda$-strongly convex.
	\end{definition}
	
	\subsubsection{One-dimensional Lower Bounds}
	We begin by proving two useful lemmas; the first will allow us to compute the error of a linear approximation of one-dimensional functions on arbitrary intervals, and the second will allow us to infer bounds on the entire domain of approximation, from the lower bounds we have on small intervals where the approximating function is linear.
	\begin{lemma}\label{lem:coefficient_lbound}
		Let $ f\in C^{2} $. Then the error of the optimal linear approximation of $ f $ denoted $ Pf $ on the interval $ \pcc{a,a+\ell} $ satisfies
		\begin{equation}
		\norm{f-Pf}_{L_2}^2 = \ell\sum_{i=2}^{\infty}\frac{\tilde{a}_i^2}{2i+1}.
		\end{equation} 
	\end{lemma}
	\begin{proof}
		A standard result on Legendre polynomials is that given any function $ f $ on the interval $ \pcc{a,a+\ell} $, the best linear approximation (w.r.t. the $L_2$ norm) is given by
		\begin{equation*}
		Pf = \tilde{a}_0\tilde{P}_0\p{x}+\tilde{a}_1\tilde{P}_1\p{x},
		\end{equation*}
		where $\tilde{P}_0,\tilde{P}_1$ are the shifted Legendre polynomials of degree $0$ and $1$ respectively, and $\tilde{a}_0,\tilde{a}_1$ are the first two Fourier-Legendre coefficients of $ f $ as defined in \eqref{eq:leg_coefficients}. The square of the error obtained by this approximation is therefore
		\begin{align*}
		\norm{f-Pf}^2 &=
		\norm{f}^2 -2 \inner{f,Pf} + \norm{Pf}^2\\
		&= \ell\p{\sum_{i=0}^{\infty}\frac{\tilde{a}_i^2}{2i+1} - 2\p{\tilde{a}_0^2+\frac{\tilde{a}_1^2}{3}} +\tilde{a}_0^2+\frac{\tilde{a}_1^2}{3}}\\
		&= \ell\sum_{i=2}^{\infty}\frac{\tilde{a}_i^2}{2i+1}.
		\end{align*}
		Where in the second equality we used the orthogonality relationship from \eqref{eq:shift_ortho}, and the generalized Parseval's identity from \eqref{eq:gen_parseval}.
	\end{proof}
	
	\begin{lemma}\label{lem:interval_to_domain_lbound}
		Suppose $ f:\pcc{0,1}\to\bbr $ satisfies $ \norm{f-Pf}_{L_2}^2\ge c\ell^5 $ for some constant $ c>0 $, and on any interval $ \pcc{a,a+\ell}\subseteq\pcc{0,1} $. Then
		\begin{equation*}
		\inf_{g\in\mathcal{G}_n}\int_0^1 (f-g)^2 d\mu \ge \frac{c}{n^4}.
		\end{equation*}
	\end{lemma}
	
	\begin{proof}
		Let $ g\in\mathcal{G}_n $ be some function, let $ a_0=0,a_1,\dots,a_{n-1},a_n=1 $ denote its partition into segments of length $ \ell_j=a_j-a_{j-1}$, where $ g $ is linear when restricted to any interval $ \pcc{a_{j-1},a_j} $, and let $ g_j,~j=1,\dots,n $ denote the linear restriction of $ g $ to the interval $ \pcc{a_{j-1},a_j} $. Then
		\begin{align}
		\int_0^1\p{f-g}^2d\mu
		&= \sum_{j=1}^n \int_{a_{j-1}}^{a_j}\p{f-g_j}^2d\mu \nonumber\\
		&\ge \sum_{j=1}^n c\ell_j^5 \nonumber\\
		&= c\sum_{j=1}^n \ell_j^5. \label{eq:sum_lbound}
		\end{align}
		Now, recall H\"{o}lder's sum inequality which states that for any $ p,q $ satisfying $ \frac{1}{p}+\frac{1}{q}=1 $ we have
		\begin{equation*}
		\sum_{j=1}^n\abs{x_jy_j}\le \p{\sum_{j=1}^n\abs{x_j}^p}^{1/p} \p{\sum_{j=1}^n\abs{y_j}^q}^{1/q}.
		\end{equation*}
		Plugging in $ x_j=\ell_j,y_j=1~~\forall j\in\set{1,\dots,n} $ we have
		\begin{equation*}
		\sum_{j=1}^n\abs{\ell_j}\le \p{\sum_{j=1}^n\abs{\ell_j}^p}^{1/p} n^{1/q},
		\end{equation*}
		and using the equalities $ \sum_{j=1}^{n} \abs{\ell_j}=1$ and $ \frac{p}{q}=p-1 $ we get that
		\begin{equation}\label{eq:holder_ineq}
		\frac{1}{n^{p-1}}\le \sum_{j=1}^n\abs{\ell_j}^p.
		\end{equation}
		Plugging the inequality from \eqref{eq:holder_ineq} with $ p=5 $ in \eqref{eq:sum_lbound} yields
		\begin{equation*}
		\int_0^1\p{p-g}^2d\mu \ge \frac{c}{n^4},
		\end{equation*}
		concluding the proof of the lemma.
	\end{proof}
	
	Our first lower bound for approximation using piece-wise linear functions is for non-linear target functions of the simplest kind. Namely, we obtain lower bounds on quadratic functions.
	\begin{theorem}\label{thm:quad_lbound}
		If $ \mathcal{G}_n $ is the family of piece-wise linear functions with at most $ n $ linear segments in the interval $ \pcc{0,1} $, then for any quadratic function $ p(x)=p_2x^2+p_1x+p_0 $, we have
		\begin{equation}\label{eq:quad_lbound}
		\inf_{g\in\mathcal{G}_n}\int_0^1 (p-g)^2 d\mu \ge \frac{p_2^2}{180n^4}.
		\end{equation}
	\end{theorem}
	
	\begin{proof}
		Observe that since $ p $ is a degree $ 2 $ polynomial, we have that its coefficients satisfy $ \tilde{a}_i=0~\forall i\ge3 $, so from \lemref{lem:coefficient_lbound} its optimal approximation error equals $ \frac{\tilde{a}_2^2\ell}{5} $. Computing $ \tilde{a}_2 $ can be done directly from the equation
		\begin{equation*}
		p\p{x} = \sum_{i=0}^{2}\tilde{a}_i\tilde{P}_i\p{x},
		\end{equation*}
		Which gives $ \tilde{a}_2=\frac{p_2\ell^2}{6} $ due to \eqref{eq:P2}. This implies that
		\begin{equation*}
		\norm{p-Pp}^2=\frac{p_2^2\ell^5}{180}.
		\end{equation*}
		Note that for quadratic functions, the optimal error is dependent solely on the length of the interval.
		Using \lemref{lem:interval_to_domain_lbound} with $ c=\frac{p_2^2}{180} $ we get
		\begin{equation*}
		\int_0^1\p{p-g}^2d\mu \ge \frac{p_2^2}{180n^4},
		\end{equation*}
		concluding the proof of the theorem.
	\end{proof}
	
	Computing a lower bound for quadratic functions is made easy since the bound on any interval $ \pcc{a,a+\ell} $ depends on $ \ell $ but not on $ a $. This is not the case in general,  as can be seen by observing monomials of high degree $ k $. As $ k $ grows, $ x^k $ on the interval $ \pcc{0,0.5} $ converges rapidly to $ 0 $, whereas on $ \pcc{1-\frac{1}{k},1} $ its second derivative is lower bounded by $ \frac{k\p{k-1}}{4} $, which indicates that indeed a lower bound for $ x^k $ will depend on $ a $.
	
	For non-quadratic functions, however, we now show that a lower bound can be derived under the assumption of strong convexity (or strong concavity) in $ \pcc{0,1} $.
	
	\begin{theorem}\label{thm:str_convex_lbound}
		Suppose $ f:\pcc{0,1} \to \bbr $ is $C^2$ and either $ \lambda $-strongly convex or $\lambda$-strongly concave. Then
		\begin{equation}
		\inf_{g\in\mathcal{G}_n}\int_0^1 (f-g)^2 d\mu \ge c\lambda^2n^{-4},
		\end{equation}
		where $ c>0 $ is a universal constant.
	\end{theorem}
	
	\begin{proof}
		We first stress that an analogous assumption to $ \lambda $-strong convexity would be that $ f $ is $ \lambda $-strongly concave, since the same bound can be derived under concavity by simply applying the theorem to the additive inverse of $ f $, and observing that the additive inverse of any piece-wise linear approximation of $ f $ is in itself, of course, a piece-wise linear function. For this reason from now on we shall use the convexity assumption, but will also refer without loss of generality to concave functions.
		
		%\note{Make sure proof indeed also explains why the result holds for strong concavity}
		
		As in the previous proof, we first prove a bound on intervals of length $ \ell $ and then generalize for the unit interval.
		From \lemref{lem:coefficient_lbound}, it suffices that we lower bound $ \tilde{a}_2 $ (although this might not give the tightest lower bound in terms of constants, it is possible to show that it does give a tight bound over all $C^2$ functions). %\note{Itay wanted to elaborate}.
		We compute
		\begin{align*}
		\tilde{a}_2 &= \frac{5}{\ell}\int_a^{a+\ell} \tilde{P}_2\p{x}f\p{x}dx \\
		&= \frac{5}{\ell}\int_a^{a+\ell} P_2\p{\frac{2}{\ell}x-\frac{2}{\ell}a-1 }f\p{x}dx,
		\end{align*}
		using the change of variables $ t= \frac{2}{\ell}x-\frac{2}{\ell}a-1 $, $ dt=\frac{2}{\ell}dx $, we get the above equals
		\begin{align*}
		&\frac{5}{2}\int_{-1}^1 P_2\p{t}f\p{\frac{\ell}{2}t+\frac{\ell}{2}+a}dt \\
		=& \frac{5}{4}\int_{-1}^1 \p{3t^2-1}f\p{\frac{\ell}{2}t+\frac{\ell}{2}+a}dt.
		\end{align*}
		We now integrate by parts twice, taking the anti-derivative of the polynomial to obtain
		\begin{align}
		& \frac{5}{4}\int_{-1}^1 \p{3t^2-1}f\p{\frac{\ell}{2}t+\frac{\ell}{2}+a}dt \nonumber\\
		=& \frac{5}{4} \pcc{\p{t^3-t}f\p{\frac{\ell}{2}t+\frac{\ell}{2}+a}}_{-1}^1 - \frac{5\ell}{8}\int_{-1}^1 \p{t^3-t}f^{\prime}\p{\frac{\ell}{2}t+\frac{\ell}{2}+a}dt \nonumber\\
		=& \frac{5\ell}{8}\int_{-1}^1 \p{t-t^3}f^{\prime}\p{\frac{\ell}{2}t+\frac{\ell}{2}+a}dt \nonumber\\
		=& \frac{5\ell}{8} \pcc{\p{\frac{t^2}{2}-\frac{t^4}{4}}f^{\prime}\p{\frac{\ell}{2}t+\frac{\ell}{2}+a}}_{-1}^1 \nonumber\\
		& ~~~~~ - \frac{5\ell^2}{16}\int_{-1}^1 \p{\frac{t^2}{2}-\frac{t^4}{4}}f^{\prime\prime}\p{\frac{\ell}{2}t+\frac{\ell}{2}+a}dt \nonumber\\
		=& \frac{5\ell}{32}\p{f^{\prime}\p{a+\ell}-f^{\prime}\p{a}} - \frac{5\ell^2}{16}\int_{-1}^1 \p{\frac{t^2}{2}-\frac{t^4}{4}}f^{\prime\prime}\p{\frac{\ell}{2}t+\frac{\ell}{2}+a}dt. \label{eq:dbl_int_by_parts}
		\end{align}
		But since $ \frac{t^2}{2}-\frac{t^4}{4}\in\pcc{0,\frac{1}{4}} ~\forall t\in\pcc{-1,1}$ and since $ f^{\prime\prime}>0 $ due to strong convexity, we have that
		\begin{equation*}
		\int_{-1}^1 \p{\frac{t^2}{2}-\frac{t^4}{4}}f^{\prime\prime}\p{\frac{\ell}{2}t+\frac{\ell}{2}+a}dt \le \frac{1}{4}\int_{-1}^1 f^{\prime\prime}\p{\frac{\ell}{2}t+\frac{\ell}{2}+a}dt.
		\end{equation*}
		Plugging this inequality in \eqref{eq:dbl_int_by_parts} yields
		\begin{align*}
		\tilde{a}_2 &\ge \frac{5\ell}{32}\p{f^{\prime}\p{a+\ell}-f^{\prime}\p{a}} - \frac{5\ell^2}{64}\int_{-1}^1 f^{\prime\prime}\p{\frac{\ell}{2}t+\frac{\ell}{2}+a}dt \\
		&= \frac{5\ell}{32}\p{f^{\prime}\p{a+\ell}-f^{\prime}\p{a}} - \frac{5\ell^2}{64} \p{f^{\prime}\p{a+\ell}-f^{\prime}\p{a}} \\
		&= \p{1-\frac{\ell}{2}} \frac{5\ell}{32}\p{f^{\prime}\p{a+\ell}-f^{\prime}\p{a}},
		\end{align*}
		but $ \ell \le 1 $, so the above is at least
		\begin{equation}\label{eq:lagrange_str_convex}
		\frac{5\ell}{64}\p{f^{\prime}\p{a+\ell}-f^{\prime}\p{a}}.
		\end{equation}
		By Lagrange`s intermediate value theorem, there exists some $ \xi\in\pcc{a,a+\ell} $ such that $ f^{\prime}\p{a+\ell}-f^{\prime}\p{a} =\ell f^{\prime\prime}\p{\xi} $, so \eqref{eq:lagrange_str_convex} is at least
		\begin{equation*}
		\frac{5\ell^2}{64}f^{\prime\prime}\p{\xi},
		\end{equation*}
		and by using the strong convexity of $f$ again, we get that
		\begin{equation*}
		\tilde{a}_2\ge \frac{5\lambda\ell^2}{64}.
		\end{equation*}
		\lemref{lem:coefficient_lbound} now gives
		\begin{align*}
		\norm{f-Pf}^2 &= \ell\sum_{i=2}^{\infty}\frac{\tilde{a}_i^2}{2i+1} \\
		&\ge \ell\frac{\tilde{a}_2^2}{5} \\
		&\ge \frac{5\lambda^2\ell^5}{4096}.
		\end{align*} 
		Finally, by using \lemref{lem:interval_to_domain_lbound} we conclude
		\begin{equation*}
		\inf_{g\in\mathcal{G}_n}\int_0^1 (f-g)^2 d\mu \ge \frac{5\lambda^2}{4096n^4}.
		\end{equation*}
	\end{proof}
	
	We now derive a general lower bound for functions $ f:\pcc{0,1}\to\bbr $. 
	
	\begin{theorem}\label{thm:1d_lbound}
		Suppose $ f:\pcc{0,1}\to\bbr $ is $ C^2 $. Then for any $ \lambda>0 $
		\begin{equation*}
		\inf_{g\in\mathcal{G}_n}\int_0^1 (f-g)^2 d\mu \ge \frac{c\cdot\lambda^2\cdot\sigma_\lambda(f)^5}{n^4}.
		\end{equation*}
	\end{theorem}
	
	\begin{proof}
		First, observe that if $ f $ is $ \lambda $-strongly convex on $ \pcc{a,b} $, then $ f\p{\p{b-a}x+a} $ is $ \lambda\p{b-a}^2 $-strongly convex on $ \pcc{0,1} $ since $ \forall x\in\pcc{0,1} $, $$ \frac{\partial}{\partial x^2}f\p{\p{b-a}x+a} = \p{b-a}^2f^{\prime\prime}\p{\p{b-a}x+a} \ge \lambda\p{b-a}^2.$$
		Now, we use the change of variables $ x = \p{b-a}t+a $, $ dx=\p{b-a}dt $
		\begin{align}\label{eq:ab_change_var}
		&\inf_{g\in\mathcal{G}_n}\int_a^b (f(x)-g(x))^2 dx \nonumber\\
		=& \inf_{g\in\mathcal{G}_n}\p{b-a}\int_0^1 \p{f\p{\p{b-a}t+a}-g\p{\p{b-a}t+a}}^2 dt \nonumber\\
		=& \inf_{g\in\mathcal{G}_n}\p{b-a}\int_0^1 \p{f\p{\p{b-a}t+a}-g\p{t}}^2 dt \nonumber\\
		\ge& \frac{c\cdot\lambda^2\cdot\p{b-a}^5}{n^4},
		\end{align}
		where the inequality follows from an application of \thmref{thm:str_convex_lbound}.
		Back to the theorem statement, if $ \sigma_\lambda=0 $ then the bound trivially holds, therefore assume $ \lambda>0 $ such that $ \sigma_\lambda > 0$. Since $ f $ is strongly convex on a set of measure $ \sigma_\lambda >0$, the theorem follows by applying the inequality from \eqref{eq:ab_change_var}.
	\end{proof}
	
	\subsubsection{Multi-dimensional Lower Bounds}
	We now move to generalize the bounds in the previous subsection to general dimension $ d $. Namely, we can now turn to proving \thmref{thm:main_lbound}.

	\begin{proof}[Proof of \thmref{thm:main_lbound}]
		Analogously to the proof of \thmref{thm:1d_lbound}, we identify a neighborhood of $ f $ in which the restriction of $ f $ to a line in a certain direction is non-linear. We then integrate along all lines in that direction and use the result of \thmref{thm:1d_lbound} to establish the lower bound.
		
		Before we can prove the theorem, we need to assert that indeed there exists a set having a strictly positive measure where $ f $ has strong curvature along a certain direction. Assuming $ f $ is not piece-wise linear; namely, we have some $ \bx_0\in\pcc{0,1}^d $ such that $ H\p{f}\p{\bx_0}\neq0 $. Since $ H\p{f} $ is continuous, we have that the function $ h_{\bv}\p{\bx}=\bv^{\top}H\p{f}\p{\bx}\bv $ is continuous and there exists a direction $ \bv\in\mathbb{S}^{d-1} $ where without loss of generality $ h_{\bv}\p{\bx_0}>0 $. Thus, we have an open neighborhood containing $ \bx_0 $ where restricting $ f $ to the direction $ \bv $ forms a strongly convex function, which implies that indeed $ \sigma_\lambda>0 $ for small enough $ \lambda>0 $.
		
		We now integrate the approximation error on $ f $ in the neighborhood $ U $ along the direction $ \bv $. Compute
		
		\begin{align*}
		&\inf_{g\in\mathcal{G}_n^d}\int_{\pcc{0,1}^d} (f-g)^2 d\mu_d \\
		=& \inf_{g\in\mathcal{G}_n^d}\int_{\bu:\inner{\bu,\bv}=0} \int_{\beta:\p{\bu+\beta\bv}\in \pcc{0,1}^d}\p{f-g}^2d\mu_1d\mu_{d-1} \\
		\ge& \inf_{g\in\mathcal{G}_n^d}\int_{\bu:\inner{\bu,\bv}=0} \int_{\beta:\p{\bu+\beta\bv}\in U}\p{f-g}^2d\mu_1d\mu_{d-1} \\
		\ge& \int_{\bu:\inner{\bu,\bv}=0} \p{\mu_1\p{\set{\beta:\p{\bu+\beta\bv}\in U}}}^5\frac{5\lambda^2}{4096n^4}d\mu_{d-1} \\
		=& \frac{5\lambda^2}{4096n^4}\int_{\bu:\inner{\bu,\bv}=0}\abs{\mu_1\p{\set{\beta:\p{\bu+\beta\bv}\in U}}}^5d\mu_{d-1} \\
		\ge& \frac{5\lambda^2}{4096n^4} \p{\int_{\bu:\inner{\bu,\bv}=0} \mu_1\p{\set{\beta:\p{\bu+\beta\bv}\in U}}d\mu_{d-1}}^5 \\
		=& \frac{5\lambda^2\sigma_\lambda^5}{4096n^4},
		\end{align*}
		where in the second inequality we used \thmref{thm:1d_lbound} and in the third inequality we used Jensen`s inequality with respect to the convex function $ x\mapsto\abs{x}^5 $.
	\end{proof}

	\subsection{Proof of \thmref{thm:main_ubound}}\label{subsec:main_ubound_proof}

		We begin by monitoring the rate of growth of the error when performing 
		either an addition or a multiplication. Suppose that the given input $ 
		\tilde{a},\tilde{b} $ is of distance at most $ \delta>0 $ from the 
		desired target values $ a,b $, i.e., $ \abs{a-\tilde{a}}\le\delta, 
		|b-\tilde{b}|\le\delta $. Then for addition we have
		\begin{equation*}
		\abs{\p{a+b}-\p{\tilde{a}+\tilde{b}}}\le \abs{a-\tilde{a}} + 
		\abs{b-\tilde{b}} \le 2\delta,
		\end{equation*}
		and for multiplication we compute the product error estimation
		\begin{align}
		\abs{\tilde{a}\cdot\tilde{b}-a\cdot b} &\le \abs{\p{a+\delta}\cdot 
		\p{b+\delta}-a\cdot b} \nonumber\\
		&=\abs{\delta\p{a+b}+\delta^2}. \label{eq:prod_err}
		\end{align}
		Now, we have bounded the error of approximating the product of two numbers which we only have approximations of, but since the computation of the product itself cannot be done with perfect accuracy using ReLU networks, we need to suffer the error of approximating a product, as shown in \thmref{thm:product_approx}. We add the error of approximating 
		the product of $ \tilde{a}\cdot \tilde{b} $, which we may assume is at 
		most $ \delta $ (assuming $ \Theta\p{\log_2\p{M/\delta}} $ bits are 
		used for the product, since each intermediate computation is bounded 
		in 
		the interval $ \pcc{-M,M} $). Overall, we get an error bound 
		of%\note{Daniel said this equation is not 
		%clear, maybe clarify this further?}
		\begin{equation*}
		\abs{\delta\p{a+b}+\delta^2+\delta}\le3M\delta.
		\end{equation*}
		From this, we see that at each stage the error grows by at most a 
		multiplicative 
		factor of $ 3M $.
		After $ t $ operations, and with an initial estimation error of $ 
		\delta $, we have that the error is bounded by $ \p{3M}^{t-1}\delta $. 
		Choosing $ \delta \le \p{3M}^{1-t}\epsilon $ to guarantee approximation 
		$ \epsilon $, we have from \thmref{thm:product_approx} that each 
		operation will require at most
		\begin{equation*}
		4\ceil{\log\p{\frac{M\p{3M}^{t-1}}{\epsilon}}}+13\le 
		c\p{\log\p{\frac{1}{\epsilon}}+t\log\p{M}}
		\end{equation*}
		width and
		\begin{equation*}
		2\ceil{\log\p{\frac{M\p{3M}^{t-1}}{\epsilon}}}+9\le 
		c\p{\log\p{\frac{1}{\epsilon}}+t\log\p{M}}
		\end{equation*}
		depth for some universal $ c>0 $. Composing the networks performing 
		each operation, we arrive at a total network width and depth of at most
		\begin{equation*}
		c\p{t\log\p{\frac{1}{\epsilon}}+t^2\log\p{M}}.
		\end{equation*}
		Now, our target function is approximated to accuracy $ \epsilon $ by a 
		function which our network approximates to the same accuracy $ \epsilon 
		$, for a total approximation error of the target function by our 
		network of $ 2\epsilon $.

\end{document}